\documentclass[10pt,twoside]{IEEEtran}
\usepackage{amsopn}
\usepackage{amsthm}
\usepackage{etex}
\usepackage{graphicx}
\usepackage{endnotes}
\usepackage{hyperref}
\usepackage{epsfig,psfrag}
\usepackage{pst-all}
\usepackage{amssymb,amsfonts,upref,cite,epsf,color,bm}
\usepackage{graphicx}
\usepackage{color}
\usepackage{csvsimple}
\usepackage{url}
\usepackage{amsmath}
\usepackage{graphicx}
\usepackage{calc}
\usepackage{booktabs}
\usepackage{tikz,stackengine}
\usepackage{pgfplots}
\usetikzlibrary{calc, shapes, fit, positioning}
\graphicspath{../images/}

\newcommand\defeq{:=}

\DeclareMathOperator*{\argmin}{\mbox{arg\;min}}
\DeclareMathOperator*{\argmaximum}{\mbox{arg\;max}}

\newcommand\vect[1]{\mathbf #1}

\newcommand{\vx}{\vect{x}}  
\newcommand{\vy}{\vect{y}}  
\newcommand{\vz}{\vect{z}}

\newcommand{\mA}{\mathbf{A}}
\newcommand{\mD}{\mathbf{D}}

\newcommand{\mW}{\mathbf{W}}

\newcommand{\signalsize}{N}
\newcommand{\maxiter}{K}

\newcommand{\graphsigs}{\mathbb{R}^{\signalsize}} 
\newcommand{\edgesigs}{\mathbb{R}^{\nredges}}

\newcommand{\edges}{\mathcal{E}}
\newcommand{\nrnodes}{\signalsize}
\newcommand{\nredges}{E}

\newcommand{\nodes}{\mathcal{V}}
\newcommand{\cluster}{\mathcal{C}}
\newcommand{\graph}{\mathcal{G}}
\newcommand{\trainingset}{\mathcal{M}}
\newcommand{\boundary}{\partial \mathcal{F}}
\newcommand{\partition}{\mathcal{F}}
\newcommand{\xsig}{\vx}
\newcommand{\xsigval}[1]{x_{{#1}}}

\newcommand{\mU}{\mathbf{U}}
\newcommand{\mQ}{\mathbf{Q}}

\newcommand{\mS}{\mathbf{S}}

\newcommand{\primslp}{\hat{\vx}}
\newcommand{\dualslp}{\hat{\vy}}
\usepackage{algorithm}
\usepackage{algpseudocode}
\floatname{algorithm}{Algorithm}
\algnewcommand\algorithmicinput{\textbf{Input:}}
\algnewcommand\INPUT{\item[\algorithmicinput]}
\algnewcommand\algorithmicoutput{\textbf{Output:}}
\algnewcommand\OUTPUT{\item[\algorithmicoutput]}
\newtheorem{theorem}{Theorem}
\newtheorem{proposition}{Proposition}
\newtheorem{definition}[theorem]{Definition}

\newtheorem{corollary}[theorem]{Corollary}

\title{Semi-supervised Learning in Network-Structured Data via Total Variation Minimization}
\author{\hspace*{-10mm} Alexander Jung$^{1}$, Alfred O. Hero III$^2$, Alexandru Mara$^{3}$,
 Saeed Jahromi$^{4}$, Ayelet Heimowitz$^{5}$, Yonina C. Eldar$^{6}$
 \thanks{Parts of this work have been presented in the conference paper \cite{NNSPSampta2017}.}
}

\begin{document}
\pagestyle{plain} 

\maketitle

\begin{abstract}
We provide an analysis and interpretation of total variation (TV) minimization for semi-supervised learning 
from partially-labeled network-structured data. Our approach exploits an intrinsic duality between TV minimization and 
network flow problems. In particular, we use Fenchel duality to establish a precise equivalence of TV minimization and 
a minimum cost flow problem. This provides a link between modern convex optimization methods for 
non-smooth Lasso-type problems and maximum flow algorithms. We show how a primal-dual method for TV minimization 
can be interpreted as distributed network optimization. Moreover, we derive a novel condition on the network structure and available 
label information that ensures that TV minimization accurately learns (approximately) piece-wise constant graph signals. 
This condition depends on the existence of sufficiently large network flows between labeled data points. 
We verify our analysis in numerical experiments. 

\end{abstract}


\section{Introduction}
\label{sec_intro}
We consider machine learning using partially labeled network-structured datasets that arise in signal processing \cite{Chen2015}, 
image processing \cite{ShiMalik2000}, social networks, internet and bioinformatics \cite{NewmannBook,SemiSupervisedBook}. 
Such data can be described by an ``empirical graph,'' whose nodes represent individual data points that are connected by edges 
if they are ``similar'' in an application-specific sense. The notion of similarity can be based on physical proximity (in time or space), 
physical connection (communication networks), or statistical dependency (probabilistic graphical models) \cite{LauritzenGM,BishopBook,koller2009probabilistic}. 

Besides graph structure, datasets carry additional information in the form of labels associated with individual data points. In a social network, we might 
define the personal preference for some product as the label associated with a data point (user profile). Acquiring labels is often costly 
and requires manual labor or experiment design. Therefore, we assume to have access to the labels of only a few data points of a small 
``training set.'' This paper aims at learning or recovering the labels of all data points based on the knowledge of the labels of only a few 
data points. 

Network models lend naturally to scalable algorithms via message passing over the empirical graph \cite{DistrOptStatistLearningADMM}. 
Moreover, semi-supervised learning (SSL) methods borrow statistical strength between connected data points to overcome the absence of 
label information \cite{SemiSupervisedBook}. Indeed, many SSL methods rely on a cluster assumption: labels of close-by data points are 
similar \cite{SemiSupervisedBook,belkin2004regularization,NSZ09,elalaoui16}. 
This assumption is at the heart of many successful methods in graph signal processing \cite{ChenClustered2016}, 
imaging \cite{pock_chambolle}, trend filtering \cite{Wang2016}, anomaly detection \cite{fan2018}, information retrieval \cite{Kurland2014}, 
and social networks \cite{NewmannBook}. We implement this cluster assumption by treating the labels of data points as graph signals 
with a small TV, which is the sum of the absolute values of signal differences along the edges in the empirical graph. This turns SSL into a TV minimization problem \cite{NNSPSampta2017,NNSPFrontiers2018,ComplexitySLP2018,Wang2016,pock_chambolle}. 

TV minimization problems in grid-structured image data have been 
studied in \cite{pock_chambolle,pmlr-v49-huetter16}. 
For arbitrary networks, \cite{Wang2016} studied the statistical properties of TV minimization 
when applied to noisy but fully observed labels. Considering partially labeled datasets (with arbitrary 
network structure), \cite{NNSPFrontiers2018,NNSPSampta2017,WhenIsNLASSO} offer sufficient conditions 
on the network structure and label information such that TV minimization accurately learns the labels of all data points. 
These conditions are somewhat difficult to verify, as they involve the (unknown) cluster structure of 
the empirical graph. We present a novel condition, which can be verified by 
network flow algorithms (see Section \ref{sec_main_results} and \ref{sec_two_cluster}), ensuring 
TV minimization to accurately learn labels that form a piece-wise constant graph signal.  

The cluster assumption used in this paper is different from the smoothness assumption widely used in 
graph signal processing \cite{belkin2004regularization,SemiSupervisedBook}. The smoothness assumption 
requires connected nodes to have similar labels by forcing them to live in a small subspace spanned by a few 
eigenvectors of the graph Laplacian. In contrast, the cluster assumption allows the labels to vary significantly 
over edges between two different clusters (see Section \ref{sec_TV_min} for more details). 

While minimizing TV as well as minimizing the Laplacian quadratic form are both special 
cases of $p$-Laplacian minimization \cite{elalaoui16,kyng2015algorithms}, their statistical and computational 
properties are quite different. While the Laplacian quadratic form is a smooth convex function, the TV is a 
non-smooth convex function that requires more advanced optimization techniques such as proximal 
methods \cite{pock_chambolle,ProximalMethods}. Statistically, TV-based learning may be accurate in cases 
where the Laplacian quadratic form minimizer fails. 

We analyze TV minimization using a variant of the nullspace property which provides 
necessary and sufficient conditions for the success of $\ell_{1}$ based methods \cite{EldarKutyniokCS,RauhutFoucartCS,KabRau2015Chap}. 
In a similar spirit \cite{ZhaoKaba2018} studies recovery of sparse signals defined on the edges of 
the empirical graph. In contrast, we study piece-wise constant signals defined on nodes. 

This paper continues our studies \cite{ComplexitySLP2018,NNSPSampta2017,NNSPFrontiers2018} of 
statistical and computational aspects of SSL via TV regularization. The central theme of this paper is the 
duality between TV minimization and network flow problems. The relation between network flow problems and 
energy minimization has been studied mainly for discrete-valued graph signals \cite{Goldfarb2009,Chambolle2005,Kolmogorov2004}. 
However, it is not obvious how to generalize these methods to real-valued graph signals. 

It turns out that the duality between TV minimization and network flow problems can be established in an elegant 
fashion using the concept of convex conjugate functions. This duality allows us to  
apply efficient convex optimization methods for TV minimization (see Alg. \ref{alg_sparse_label_propagation_centralized}) 
to solve network flow problems and, in the other direction, unleashes existing network-flow algorithms \cite{BertsekasNetworkOpt} 
for TV minimization. 

Our detailed contributions are:  
\begin{itemize}
\item Our main result is Proposition \ref{prop_dual_TV_min_flow}, which states that 
the dual of TV minimization is equivalent to a minimum-cost network flow problem (see Section \ref{sec_dual_TV_Min}). 
\item An immediate consequence is Corollary \ref{cor_flow_satur_constant}, which 
characterizes the solutions of TV minimization. In contrast to our previous work, Corollary \ref{cor_flow_satur_constant} does not involve 
any signal model, such as piece-wise constant signals. 
\item We provide a novel interpretation of a message passing algorithm \cite[Alg.\ 2]{ComplexitySLP2018} for TV minimization
as distributed network flow optimization (see Section \ref{sec_spl_Alg}).
\item 
Proposition \ref{lem_flow_cond} provides a new condition ensuring that TV minimization is accurate. 
In contrast to previous work \cite{NNSPSampta2017,NNSPFrontiers2018}, this 
condition can be verified easily using existing network-flow algorithms (see Section \ref{sec_two_cluster}). 
\item We verify our theoretical analysis of TV minimization by several numerical experiments (see Section \ref{experimental_results}). 
\end{itemize}

{\bf Outline.}  
In Section \ref{sec_setup}, we formulate SSL for network-structured data as a convex TV minimization problem. 
We then discuss in Section \ref{sec_dual_TV_Min} how a dual problem of TV minimization can be defined. Exploiting the relation 
between TV minimization and its dual, we discuss in Section \ref{sec_spl_Alg} how to apply a particular 
instance of a proximal method \cite{ProximalMethods} to obtain a solution to TV minimization (and its dual). 
As detailed in Section \ref{sec_spl_Alg}, the resulting algorithm can be implemented as message passing on the 
empirical graph. In Section \ref{sec_main_results}, we present a sufficient condition on the available label information 
and the empirical graph such that TV minimization delivers accurate label estimates. Numerical experiments are discussed in Section \ref{experimental_results}.


\section{Problem Formulation}
\label{sec_setup}
We formalize SSL with network-structured data as an optimization problem. Section \ref{sec_emp_graph}  
introduces relevant concepts of graph theory. Section \ref{sec_cluster_assumption} introduces the cluster 
assumption using graph signals with a small TV. A particular class of such graph signals 
is constituted by piece-wise constant graph signals as defined in Section \ref{sec_cluster_assumption}. 
The cluster assumption leads naturally to a formulation of SSL as a TV minimization problem, which 
we define and discuss in Section \ref{sec_TV_min}. 

Let us fix some notation. Given a vector $\vx\!=\!(x_{1},\ldots,x_{n})^{T}$, 
we define the norms $\| \vx \|_{1} \defeq \sum_{l=1}^{n} |x_{l}|$ 
and $\| \vx \|_{\infty}\!\defeq\!\max_{i=1,\ldots,n} |x_{i}|$. The signum ${\rm sign } \{ \vx \}$ of a vector $\vx\!=\!\big(x_{1},\ldots,x_{d}\big)$ 
is the vector $\big({\rm sign } (x_{1}),\ldots,{\rm sign } (x_{d}) \big)\!\in\!\mathbb{R}^{d}$ 
with ${\rm sign } (x_{i})\!=\!1$ for $x_{i}\!>\!1$, ${\rm sign } (x_{i})\!=\!-1$ for $x_{i}\!\leq\!0$. 

The spectral norm of a matrix $\mA$ is denoted $\| \mA \|_{2} \defeq \sup_{\| \vx \|_{2}=1} \|\mA \vx \|_{2}$. 
For a positive semidefinite (psd) matrix $\mQ \in \mathbb{R}^{n \times n}$, with spectral decomposition 
$\mQ\!=\!\mU \mS \mU^{T}$ with the diagonal matrix $\mS = {\rm diag}  \{ s_{i} \}_{i=1}^{n}$. The square 
root of psd $\mQ$ is $\mQ^{1/2}\!\defeq\!\mU \mS^{1/2} \mU^{T}$ with $\mS^{1/2}\!\defeq\!{\rm diag} \{ \sqrt{s_{i}} \}_{\!i=\!1}^{n}$. 
For a given psd $\mathbf{Q}$ we define the norm $\| \vx \|_{\mathbf{Q}} \defeq \sqrt{ \vx^{T} \mQ \vx}$. 

The subdifferential of a function $g(\vx)$ at $\vx_{0}\!\in\!\mathbb{R}^{n}$ is 
\begin{equation} 
\partial g(\vx_{0})\!\defeq\!\{ \vy\!\in\!\mathbb{R}^{n}\!:\!g(\vx)\!\geq\!g(\vx_{0})\!+\!\vy^{T}(\vx\!-\!\vx_{0}) \mbox{ for any } \vx \}, \nonumber 
\end{equation}  
and its convex conjugate function is defined as \cite{BoydConvexBook}
\begin{equation}
\label{equ_def_convex_conjugate}
g^{*}(\hat{\vy}) \defeq \sup_{\vy \in \mathbb{R}^{n}} \vy^{T}\hat{\vy}- g( \vy). 
\end{equation} 

\subsection{The Empirical Graph}
\label{sec_emp_graph}

Consider a dataset of $\signalsize$ data points (a graph signal) that can be represented as 
supported at the nodes of a simple undirected weighted graph $\graph = (\nodes, \edges, \mathbf{W})$, 
where $\nodes$ are nodes, $\edges$ are edges and $\mathbf{W}$ are edge weights. Following 
\cite{SemiSupervisedBook}, we refer to the graph $\graph$ as the empirical graph associated with 
the dataset. 

The nodes $i\!\in\!\nodes\!=\!\{1,\ldots,\signalsize\}$ of the empirical graph $\graph$ represent the 
$\signalsize$ individual data points. In many applications, the goal is to determine (or infer) some relevant property 
encoded as a numeric label $x_{i}$ associated with the node $i \in \nodes$. The labels could represent instantaneous 
amplitudes of an audio signal, the greyscale values of image pixels, or the probabilities of social network members taking a particular action. 
The labels $x_{i}$ define a graph signal $\vx = (x_{1},\ldots,x_{\signalsize})^{T} \in \mathbb{R}^{\signalsize}$ over the empirical 
graph with the signal value at node $i$ given by the label $x_{i}$. 

The undirected edges $\{i,j\}\!\in\!\edges$ of the empirical graph $\graph$ connect data points 
which are considered similar (in some domain-specific sense). It will be convenient to represent 
the edges by the numbers $\{1,\ldots,E=|\edges|\}$. 
 
For an edge $\{i,j\}\!\in\!\edges$, the nonzero value $W_{i,j}\!>\!0$ represents the strength of the connection 
$\{i,j\}\!\in\!\edges$. The edge set $\edges$ is encoded in the non-zero pattern of the weight matrix 
$\mathbf{W}\!\in\!\mathbb{R}^{\signalsize\times \signalsize}$,  
\begin{equation}
\label{equ_edge_set_support_weights}
\{ i , j \} \in \edges \mbox{ if and only if } W_{i,j}  > 0.  
\end{equation} 

The neighborhood $\mathcal{N}(i)$ and weighted degree (strength) $d_{i}$ of node $i \in \nodes$ are defined, respectively, as 
\begin{equation} 
\label{equ_def_neighborhood}
\mathcal{N}(i) \defeq \{ j \in \nodes : \{i,j\} \!\in\!\edges \} \mbox{, } d_{i} \defeq \sum_{j \in \mathcal{N}(i)} W_{i,j}.
\end{equation} 
The maximum (weighted) node degree is
\begin{equation}
\label{equ_def_max_node_degree}
 d_{\rm max} \defeq \max_{i \in \mathcal{V}} d_{i} \stackrel{\eqref{equ_def_neighborhood}}{=} \max_{i \in \nodes} \sum_{j \in \mathcal{N}(i)} W_{i,j} . 
\end{equation}
Without loss of generality we consider only datasets whose empirical graph does not contain isolated nodes, i.e., we assume that $d_{i} > 0$ for every node $i \in \nodes$.

For a given undirected empirical graph $\graph=(\nodes,\edges,\mathbf{W})$, 
we orient the undirected edge $\{i,j\}$ by defining the head as $e^{+}\!=\!\min \{i,j\}$ 
and the tail as $e^{-}\!=\!\max \{i,j\}$. The undirected edge $\{i,j\}$ with nodes $i < j$ 
becomes the directed edge $(i,j)$. We use $\graph$ and $\edges$ 
to also denote the oriented empirical graph and its directed edges, respectively.  
The \emph{incidence matrix} $\mD \!\in\! \mathbb{R}^{\nredges \times \signalsize}$ of the empirical graph $\graph$ is 
\begin{equation}
D_{e,i} = \begin{cases}  W_{e} & \mbox{ if } i = e^{+}  \\ 
				    - W_{e} & \mbox{ if } i = e^{-}  \\ 
				    0 &  \mbox{ else.}  \label{equ_def_incidence_mtx}
				    \end{cases}
\end{equation} 
The rows of $\mD$ correspond to the edges $e\!\in\!\edges$ while the columns represent 
nodes $i\!\in\!\nodes$ of the empirical graph $\graph$. 
The row representing $e\!=\!\{i,j\}$ contains exactly 
two non-zero entries in the columns corresponding to the nodes $i,j \in \nodes$. It will be convenient to 
define the directed neighbourhoods (see \eqref{equ_def_neighborhood}) of a node $i \in \nodes$ as 
\begin{align} 
\mathcal{N}^{+}(i) & \defeq \{ j \in \nodes : \{i,j\} \!\in\!\edges, i < j \} \mbox{, and }  \nonumber \\ 
\mathcal{N}^{-}(i) & \defeq \{ j \in \nodes : \{i,j\} \!\in\!\edges, i > j \}. 
\end{align}

\subsection{Cluster Assumption}
\label{sec_cluster_assumption}

We assume that labels $x_{i}$ are known at only a few nodes $i \in \nodes$ of a (small) training 
set $\trainingset \subseteq \nodes$ (see Fig.\ \ref{fig_clustered_graph_signal}). Our goal 
is then to learn the unknown labels $x_{i}$ for all data points $i \in \nodes \setminus \trainingset$ outside 
the training set. This learning problem, which is known as SSL, translates to a graph signal recovery problem 
within our setting. 

Given the signal samples $x_{i}$ for data points $i \in \trainingset$ in the training set, we want to recover 
the entire graph signal $\vx \in \mathbb{R}^{\signalsize}$. This learning (or recovery) problem is feasible 
if the underlying graph signal $\vx$ has a known structure. As mentioned above, a particular structure is 
obtained if the labels $x_{i}$ conform with the cluster structure of the empirical graph $\graph$. Consider 
the graph signal $\vx \in \mathbb{R}^{\signalsize}$ constituted by the (mostly unknown) labels $x_{i}$ of the 
data points $i\!\in\!\nodes$. The cluster assumption requires similar signal values  $x_{i}\!\approx\!x_{j}$ 
at nodes $i,j\!\in\!\nodes$ in the same well-connected subset (cluster).

We measure the ``clusteredness'' of a graph signal $\vx$ using the weighted TV \cite{RudinNoise,Wang2016}
\begin{equation} 
\label{equ_def_TV}
\| \vx \|_{\rm TV} \defeq \sum_{\{i,j\} \in \edges} W_{i,j}  | x_{j}\!-\!x_{i}|.
\end{equation} 
As the notation indicates, $\| \vx \|_{\rm TV}$ defines a seminorm for graph signals $\vx$. It is only a seminorm 
since it is zero also for non-zero (but constant) graph signals.  
The incidence matrix $\mD$ \eqref{equ_def_incidence_mtx} of the (oriented) empirical graph $\graph$ allows us to represent the TV 
of a graph signal $\vx$ as   
\begin{equation}
\label{equ_repr_ell_1_TV}
\| \vx \|_{\rm TV}  = \| \mathbf{D} \vx \|_{1}. 
\end{equation} 
Using the TV \eqref{equ_def_TV} to guide learning (signal recovery) methods turns out to 
be useful statistically and computationally. Indeed, as we discuss below, minimizing TV results 
in labels (signals) which are constant over well-connected subsets (clusters) of data points. 
Moreover, TV minimization can be implemented as highly scalable message passing over 
the underlying empirical graph (see Alg.\ \ref{sparse_label_propagation_mp}).

The most simple model for graph signals conforming with the cluster assumption 
are piece-wise constant signals \cite{Chen2015}
\begin{equation}
\label{equ_def_clustered_signal_model}
\hspace*{-3mm}\xsigval{i}\!=\!\sum_{l=1}^{|\partition|} a_{l} \mathcal{I}_{\cluster_{l}}[i]  \mbox{ with } a_{l}\!\in\!\mathbb{R} \mbox{, } \mathcal{I}_{\cluster_{l}}[i]\!\defeq\!\begin{cases} 
1 \mbox{ for } i\!\in\!\cluster_{l} \\ 0 \mbox{ else.}  \end{cases}
\end{equation} 
The signal model \eqref{equ_def_clustered_signal_model} uses an arbitrary but fixed partition 
\begin{equation} 
\nonumber 
\partition=\big\{\cluster_{1},\ldots,\cluster_{|\partition|}\big\}
\end{equation}
constituted by disjoint clusters $\cluster_{l} \subseteq \nodes$ 
(see Fig.\ \ref{fig_clustered_graph_signal}). Our analysis will be applicable for an arbitrary choice for the 
partition underlying the signal model \eqref{equ_def_clustered_signal_model}. However, our results are 
most useful for partitions which consist of well connected clusters (see Definition \ref{def_sampling_set_resolves}). 

We emphasize that the learning algorithm we propose in Section \ref{sec_spl_Alg} does not require knowledge 
of the partition $\partition$ underlying the signal model \eqref{equ_def_clustered_signal_model}. The partition is 
only required for the analysis of the learning accuracy of this algorithm (see Section \ref{sec_main_results}). 

The signal model \eqref{equ_def_clustered_signal_model} is an idealization which crucially simplifies the analysis 
of the statistical properties of TV minimization (see Section \ref{sec_TV_min}). The graph signals arising in many 
applications will typically not be perfectly constant over clusters. However, Theorem \ref{main_thm_approx_sparse} 
remains useful as long as the data (labels) can be well approximated by a piece-wise constant graph signal \eqref{equ_def_clustered_signal_model}.  

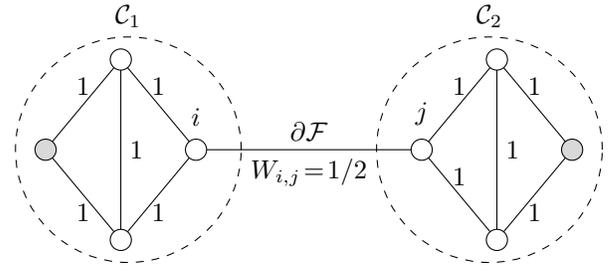
\begin{figure}[htbp]
\begin{center}
\begin{tikzpicture}
    \tikzset{x=1cm,y=1.2cm,every path/.style={>=latex},node style/.style={circle,draw}} 
    \coordinate[] (x2) at (0,0); 
    \coordinate[] (x10) at (-1,1);    
    \coordinate[] (x3) at (0,2);
    \coordinate[] (x4) at (1,1);  
    \coordinate[] (x12) at (2.5,1) ;   
    \draw[line width=0.4,-] (x2) edge  node [right] {$1$} (x3);
    \draw[line width=0.4,-] (x2) edge  node [below] {$1$} (x10);
    \draw[line width=0.4,-] (x3) edge   node [above] {$1$} (x10);
    \draw[line width=0.4,-] (x2) edge   node [below] {$1$} (x4);
    \draw[line width=0.4,-] (x3) edge   node [above] {$1$}(x4);
     \coordinate[] (x5) at (5,0); 
    \coordinate[] (x6) at (5,2);    
    \coordinate[] (x7) at (4,1);
    \coordinate[] (x11) at (6,1);          
    \draw[line width=0.4,-] (x5) edge node[right]{$1$} (x6);
        \draw[line width=0.4,-] (x11) edge node[above]{$1$}(x6);
            \draw[line width=0.4,-] (x5) edge  node [below] {$1$} (x11);
    \draw[line width=0.4,-] (x5) edge  node [above] {$1$} (x7);
    \draw[line width=0.4,-] (x6) edge  node [above] {$1$} (x7);
    \draw[line width=0.4,-] (x4) edge node [above] {$\partial \partition$}  node [below] {$W_{i,j}\!=\!1/2$}  (x7) ;
    \draw [fill=white]   (x2)  circle (4pt);
    \draw [fill=white]   (x3)  circle (4pt);
    \draw [fill=gray!30]   (x10)  circle (4pt);
    \draw [fill=white]   (x4)  circle (4pt) node [above=2mm] {$i$} ;
     \draw [fill=white]   (x5)  circle (4pt);
    \draw [fill=white]   (x6)  circle (4pt);
    \draw [fill=gray!30]   (x11)  circle (4pt);
    \draw [fill=white]   (x7)  circle (4pt) node [above=2mm] {$j$} ;
    \node[draw,circle,dashed,minimum size=3cm,inner sep=0pt,label={$\mathcal{C}_1$}] at (0.1,1) {};
    \node[draw,circle,dashed,minimum size=3cm,inner sep=0pt,label={$\mathcal{C}_2$}] at (4.9,1) {};
\end{tikzpicture}
\end{center}
\caption{\label{fig_clustered_graph_signal} Empirical graph $\graph$ whose nodes $\nodes$ are grouped into two clusters $\cluster_{1}$ and $\cluster_{2}$ 
forming the partition $\partition=\{\cluster_{1},\cluster_{2}\}$. The boundary of the partition is $\partial \partition = \{\{i,j\} \}$ having weight $W_{i,j}=1/2$. The 
edges $e \in \edges$ connecting nodes within the same cluster have weight $W_{e}=1$. The nodes belonging to the training set $\trainingset$ are shaded.}
\end{figure}

In Section \ref{sec_main_results} we characterize (see Definition \ref{def_sampling_set_resolves}) 
those partitions $\partition$, used in the model \eqref{equ_def_clustered_signal_model}, which allow for accurate 
recovery of a (approximately) piece-wise graph signal from its values $x_{i}$ at the nodes $i \in \trainingset$ of the 
training set. Our results indicate that piece-wise constant signals \eqref{equ_def_clustered_signal_model} can be 
learned accurately if the partition $\partition$ has a boundary with small weights. The boundary $\boundary$ of $\partition$ 
consists of the edges connecting nodes from different clusters, i.e., 
\begin{equation}
\boundary \defeq \{ \{i,j\} \in \edges \mbox{ with  } i \in \cluster_{l} \mbox{ and } j \in \cluster_{l'} \neq \cluster_{l} \}. \nonumber
\end{equation} 
The boundary $\boundary$ is the union of the cluster boundaries 
\begin{equation} 
\partial \cluster_{l} \defeq \{ \{i,j\} \in \edges \mbox{ with } i \in \cluster_{l} \mbox{ and } j \in \nodes \setminus \cluster_{l} \}.
\end{equation} 

Recovering a piece-wise constant graph signal \eqref{equ_def_clustered_signal_model} may seem trivial given the 
availability of efficient clustering methods \cite{Luxburg2007,Spielman_alocal,AbbeSBM2018}. Indeed, it is natural 
to first obtain the partition $\partition$ underlying \eqref{equ_def_clustered_signal_model} using some clustering method 
and then perform cluster-wise averaging in order to obtain an estimate for the coefficients $a_{l}$ in \eqref{equ_def_clustered_signal_model}. 
Despite the conceptual simplicity of this approach, it has some challenges. Most existing clustering methods  involve design 
parameters such as the number of clusters or distribution parameters of probabilistic (stochastic block) models. 
The proper choice (or learning) of these parameters can be non-trivial. Moreover, clustering methods do not exploit label information. 

In what follows, we show how the recovery problem lends naturally to a TV minimization problem which, in turn, 
can be solved by efficient convex optimization methods. The resulting algorithm (Alg.\ \ref{alg_sparse_label_propagation_centralized}) 
does not involve any design parameters and can be implemented as scalable message passing (Alg.\ \ref{sparse_label_propagation_mp}) 
on the empirical graph.

\subsection{TV Minimization}
\label{sec_TV_min} 

The TV of a piece-wise constant graph signal \eqref{equ_def_clustered_signal_model} is 
\begin{align}
\label{equ_bound_TV_norm_clustered}
\| \xsig \|_{\rm TV} & \stackrel{\eqref{equ_def_TV}}{=}  \sum_{\{i,j\} \in \edges} W_{i,j}  | \xsigval{j}\!-\!\xsigval{i}| \nonumber \\ 
&  \stackrel{\eqref{equ_def_clustered_signal_model}}{=}  \sum_{\{i,j\} \in \boundary } W_{i,j}  | \xsigval{j}\!-\!\xsigval{i}| \nonumber \\ 
&  \stackrel{\eqref{equ_def_clustered_signal_model}}{\leq}  \bigg(\sum_{\{i,j\} \in \boundary}  \hspace*{-2mm} W_{i,j}\bigg) \max_{l,l'\in\{1,\ldots,|\partition|\}} |a_{l}\!-\!a_{l'}| . 
\end{align} 
Thus, if the partition $\partition$ has a small weighted boundary $\sum_{\{i,j\} \in \boundary} W_{i,j}$, the graph signals 
\eqref{equ_def_clustered_signal_model} have a small TV $\| \vx \|_{\rm TV}$ due to \eqref{equ_bound_TV_norm_clustered}. 

A sensible strategy for learning a piece-wise constant graph signal is therefore via minimizing 
the TV $\| \tilde{\vx} \|_{\rm TV}$ among all graph signals which are consistent with 
the known labels $\{ x_{i} \}_{i \in \trainingset}$. This is formulated as the optimization problem  
\begin{align} 
\primslp &\!\in\!\argmin_{\tilde{\vx} \in \graphsigs} \underbrace{\sum_{\{i,j\} \in \edges} \hspace*{-3mm} W_{i,j}  | \tilde{x}_{j}\!-\!\tilde{x}_{i}|}_{=  \| \tilde{\vx} \|_{\rm TV}}
 \mbox{s.t. }  \tilde{x}_{i}\!=\!x_{i}  \mbox{ for all } i\!\in\!\trainingset \nonumber \\
& \stackrel{\eqref{equ_repr_ell_1_TV}}{=} 
\argmin_{\tilde{\vx} \in \graphsigs} \| \mathbf{D} \tilde{\vx} \|_{1} \quad \mbox{s.t.} \quad  \tilde{x}_{i}
\!=\!x_{i} \mbox{ for all } i\!\in\!\trainingset.  \label{equ_min_constr}
\end{align}

Since the objective function and the constraints in \eqref{equ_min_constr} are convex, the optimization 
problem \eqref{equ_min_constr} is a convex optimization problem \cite{BoydConvexBook}. In fact, 
\eqref{equ_min_constr} can be reformulated as a linear program \cite[Sec. 1.2.2]{BoydConvexBook}. 

The solution to \eqref{equ_min_constr} might not be unique.\footnote{Assume that no initial labels are 
available such that the training set $\trainingset$ would be empty. Then, every constant graph signal solves \eqref{equ_min_constr}.} 
Any such solution $\hat{\vx}$ is characterized by two properties: (i) it is consistent with the initial labels, i.e., $\hat{x}_{i}=x_{i}$ for all nodes $i\in \trainingset$ 
in the training set; and (ii) it has minimum TV among all such graph signals. 

We solve \eqref{equ_min_constr} using a recently proposed primal-dual method \cite{PrecPockChambolle2011}. 
This approach is appealing since it comes with a theoretical convergence 
guarantee and can be implemented efficiently as message passing over the underlying 
empirical graph (see Alg.\ \ref{sparse_label_propagation_mp} below). The resulting algorithm 
bears some similarity to the class of label propagation (LP) algorithms for SSL on graphs \cite{Anis2016ExpSampSet,Chen2015}. 
Indeed, LP algorithms can be interpreted as message passing methods for solving the 
optimization problem \cite[Chap 11.3.4.]{SemiSupervisedBook}: 
\begin{align}
\hat{\vx}^{(\rm LP)}&\!\in\!\argmin_{\tilde{\vx} \in \graphsigs} \sum_{\{i,j\} \in \edges} W^{2}_{i,j} (\tilde{x}_{i}\!-\!\tilde{x}_{j})^2\nonumber \\ 
  & \mbox{s.t.} \quad  \tilde{x}_{i}
\!=\!x_{i} \mbox{ for all } i\!\in\!\trainingset.  \label{equ_LP_problem}
\end{align} 

The learning problem \eqref{equ_LP_problem} amounts to minimizing the weighted sum of 
squared signal differences $(\tilde{x}_{i}-\tilde{x}_{j})^2$ over edges $\{i,j\} \in \edges$ in the 
empirical graph. In contrast, TV minimization \eqref{equ_min_constr} aims to minimize a 
weighted sum of absolute values of the signal differences $|\tilde{x}_{i}-\tilde{x}_{j}|$. It turns 
out that using the absolute values of the signal differences (the TV) instead of the sum of squared 
differences (as in LP) results in piece-wise constant graph signals (see \eqref{equ_def_clustered_signal_model}). 
In contrast, LP methods smooth out abrupt signal variations (see Section \ref{experimental_results}), 
making them unsuitable for data which can be (approximately) represented by piece-wise constant graph signals. 
LP methods have been shown to fail dramatically for random geometric graphs  \cite{NSZ09}. 

TV minimization \eqref{equ_min_constr} and LP \eqref{equ_LP_problem} are special cases of $p$-Laplacian 
minimization \cite{elalaoui16}
\begin{align}
\hat{\vx}^{(p)}&\!\in\!\argmin_{\tilde{\vx} \in \graphsigs} \sum_{\{i,j\} \in \edges} \bigg( W_{i,j} |\tilde{x}_{i}-\tilde{x}_{j}| \bigg)^{p}  \nonumber \\ 
& \mbox{s.t.} \quad  \tilde{x}_{i}\!=\!x_{i} \mbox{ for all } i\!\in\!\trainingset.  \label{equ_pLap_problem}
\end{align} 
Indeed, TV minimization \eqref{equ_min_constr} is obtained from \eqref{equ_pLap_problem} when $p\!=\!1$, while the 
LP problem \eqref{equ_LP_problem} is obtained when $p\!=\!2$. The limiting case of \eqref{equ_pLap_problem} for 
$p\rightarrow \infty$, known as the \emph{minimal Lipschitz extension problem}, is studied in \cite{kyng2015algorithms}. 
The work \cite{kyng2015algorithms} presents efficient solvers and proves stability of the solutions for \eqref{equ_pLap_problem} 
in this limiting case. However, while the algorithms in \cite{kyng2015algorithms} have high (combinatorial) complexity, 
we can solve TV minimization using efficient convex optimization methods (see Section \ref{sec_spl_Alg}). 

The TV minimization problem \eqref{equ_min_constr} is also closely related to graph trend filtering \cite{Wang2016} 
and the more general network Lasso (nLasso) \cite{NetworkLasso,WhenIsNLASSO}
\begin{equation}
\hat{\vx}^{(\rm nL)} \!\in\!  \argmin_{\tilde{\vx} \in \graphsigs} \sum_{i \in \trainingset}  (\tilde{x}_{i}\!-\!x_{i})^2 \!+\! \lambda \| \tilde{\vx}\|_{\rm TV} \label{equ_nLasso}.
\end{equation} 
By Lagrangian duality \cite{BertsekasNonLinProgr,BoydConvexBook}, there are values (which might depend on the initial labels $x_{i}$) 
for $\lambda$ in \eqref{equ_nLasso} such that solutions of \eqref{equ_nLasso} coincide with those of \eqref{equ_min_constr}. The tuning 
parameter $\lambda\!>\!0$ in \eqref{equ_nLasso} allows us to trade a small empirical error $\sum_{i \in \trainingset}  (\hat{x}^{(\rm nL)}_{i}\!-\!x_{i})^2$ against a small TV $\|\hat{\vx}^{(\rm nL)}\|_{\rm TV}$ 
of the learned graph signal $\hat{\vx}^{(\rm nL)}$. Choosing a large value of $\lambda$ enforces a small TV of the learned graph signal. 
Using a small value for $\lambda$ puts more emphasis on the empirical error. In contrast to nLasso \eqref{equ_nLasso}, TV minimization \eqref{equ_min_constr} 
does not require any parameter tuning.

\section{The Dual of TV Minimization}
\label{sec_dual_TV_Min}

TV minimization \eqref{equ_min_constr} involves non-differentiable objective function, 
which rules out gradient (descent) methods.  However, both the objective function and 
the constraint set of \eqref{equ_min_constr}  have a simple structure individually. This 
compositional structure of \eqref{equ_min_constr} can be exploited by studying an 
equivalent dual problem. It turns out that this dual problem has an interpretation as 
network (flow) optimization \cite{BertsekasNetworkOpt}.  Moreover, by jointly considered the 
primal TV minimization \eqref{equ_min_constr} and its dual we obtain an efficient method for 
simultaneously solving TV minimization \eqref{equ_min_constr} and its dual (see Section \ref{sec_spl_Alg}). 

In order to formulate the dual problem we first reformulate TV minimization \eqref{equ_min_constr} 
as an equivalent unconstrained convex optimization problem 
\begin{align}
\label{equ_min_constr_unconstr}
\primslp & \!\in\! \argmin_{\tilde{\vx} \in \mathbb{R}^{\signalsize}} f(\tilde{\vx}) \defeq g(\mD \tilde{\vx}) + h(\tilde{\vx}),
\end{align}
with 
\begin{equation}
\nonumber
 g(\vy) \defeq \| \vy\|_{1} \mbox{, and } h(\tilde{\vx}) \defeq  \begin{cases} \infty  \mbox{ if } \tilde{\vx} \notin \mathcal{Q} \\ 0 \mbox{ if } \vx \in \mathcal{Q}.\end{cases}
 \end{equation} 
The constraint set $\mathcal{Q} = \{ \tilde{\vx} \in \mathbb{R}^{\signalsize} : \tilde{x}_{i} = x_{i} \mbox{ for all } i \in \trainingset \}$ 
collects all graph signals which match the labels $x_{i}$ on the training set $\trainingset$. 
The (extended-value) function $h(\vx)$ in \eqref{equ_min_constr_unconstr} is the indicator function of the convex set $\mathcal{Q}$ (see \cite{BoydConvexBook}).  

We can view \eqref{equ_min_constr_unconstr} as the primal problem (or formulation) of TV minimization \eqref{equ_min_constr}. 
The dual problem associated with TV minimization is 
\begin{align} 
\label{equ_dual_SLP}
\dualslp & \!\in\! \argmaximum_{\vy \in \mathbb{R}^{\nredges}} \tilde{f}(\vy) \defeq -h^{*}(-\mD^{T} \vy) - g^{*}(\vy).
\end{align} 
The objective function $\tilde{f}(\vy)$ of the dual problem \eqref{equ_dual_SLP} is composed of the convex conjugates 
(see \eqref{equ_def_convex_conjugate}) of the components $h(\vx)$ and $g(\vy)$ of the primal problem \eqref{equ_min_constr_unconstr}. 
These convex conjugates are given explicitly by 
\begin{align} 
\label{equ_conv_conj_h}
h^{*} (\tilde{\vx}) & = \sup_{\vz \in \mathbb{R}^{\signalsize}}  \vz^{T} \tilde{\vx}- h(\vz) \nonumber  \\
& \hspace*{-15mm} \stackrel{\eqref{equ_min_constr_unconstr}}{=} \begin{cases} \infty &\mbox{ if } \tilde{x}_{i} \neq 0 \mbox{ for some } i\!\in\!\nodes \setminus \trainingset \\ \sum_{i \in \trainingset} \tilde{x}_{i} x_{i} & \mbox{ otherwise,} \end{cases}
\end{align} 
and 
\begin{align} 
\label{equ_conv_conj_g}
g^{*}(\vy)& = \sup_{\vz \in \mathbb{R}^{\nredges}}  \vz^{T} \vy - g(\vz)  \stackrel{\eqref{equ_min_constr_unconstr}}{=}
 \sup_{\vz \in \mathbb{R}^{\nredges}}  \vz^{T} \vy - \|\vz\|_{1} \nonumber  \\
& \stackrel{}{=} \begin{cases}  \infty & \mbox{ if } \| \vy \|_{\infty} > 1 \\ 0 & \mbox{ otherwise.} \end{cases}
\end{align} 

The relation between the primal problem \eqref{equ_min_constr_unconstr} and the dual problem \eqref{equ_dual_SLP} 
is made precise in \cite[Thm.\ 31.3]{RockafellarBook}. First, the optimal values of \eqref{equ_min_constr_unconstr} 
and \eqref{equ_dual_SLP} coincide: 
\begin{equation} 
\label{equ_zero_duality_gap}
\min_{\tilde{\vx} \in \mathbb{R}^{\signalsize}}  g(\mD \tilde{\vx}) + h(\tilde{\vx}) = \max_{\vy \in \mathbb{R}^{\nredges}}-h^{*}(-\mD^{T} \vy) - g^{*}(\vy). 
\end{equation}
The identity \eqref{equ_zero_duality_gap} is useful for bounding the sub-optimality $\| \tilde{\vx} \|_{\rm TV} - \| \hat{\vx} \|_{\rm TV}$ of a given 
candidate $\tilde{\vx}$ for the solution $\hat{\vx}$ to the TV minimization \eqref{equ_min_constr}. According to \eqref{equ_zero_duality_gap}, 
given any (dual) vector $\vy \in \mathbb{R}^{\nredges}$, we can bound the sub-optimality as 
\begin{equation}
\label{equ_upper_bound_subopt}
\| \tilde{\vx} \|_{\rm TV} - \| \hat{\vx} \|_{\rm TV} \leq  \|  \tilde{\vx} \|_{\rm TV} + \big( h^{*}(-\mD^{T} \vy) + g^{*}(\vy) \big). 
\end{equation} 
Another consequence of the duality result \cite[Thm.\ 31.3]{RockafellarBook} is a powerful characterization of the solutions of the primal \eqref{equ_min_constr_unconstr} 
and dual problem \eqref{equ_dual_SLP}. In particular, a pair of vectors $\hat{\vx} \in \graphsigs, \dualslp \in \mathbb{R}^{\nredges}$ 
are solutions to the primal \eqref{equ_min_constr_unconstr} and dual problem \eqref{equ_dual_SLP}, respectively, if and only if 
\begin{equation}
\label{equ_two_coupled_conditions}
-(\mD^{T} \dualslp) \in \partial h(\primslp) \mbox{, } \mD \primslp \in \partial g^{*}(\dualslp) . 
\end{equation} 
Given any solution $\dualslp \in \mathbb{R}^{\nredges}$ to the dual problem \eqref{equ_dual_SLP}, 
any solution $\hat{\vx}$ to the primal problem \eqref{equ_min_constr_unconstr} and, in turn, to TV minimization 
\eqref{equ_min_constr} must be such that conditions \eqref{equ_two_coupled_conditions} are satisfied. The optimality 
condition \eqref{equ_two_coupled_conditions} is the launching point for a primal-dual method for solving \eqref{equ_min_constr} (see Section \ref{sec_spl_Alg}). 

It turns out that the dual \eqref{equ_dual_SLP} of TV minimization \eqref{equ_min_constr} is an instance of 
network optimization for the empirical graph $\graph$. To show this, we need the following definition.
\begin{definition}
A network flow $f\!:\!\edges\!\rightarrow\!\mathbb{R}$ with supplies $v_{i}$, at the nodes $i\!\in\!\nodes$, assigns each 
directed edge $e\!=\!(i,j)\!\in\!\edges$ some value $f_{e}\!\in\!\mathbb{R}$. 
The flow has to satisfy the conservation law: 
\begin{equation} 
\label{equ_conservation_law}
\hspace*{-3mm}\sum_{j \in \mathcal{N}^{+}(i)} \hspace*{-2mm}f_{(i,j)}\!-\!\sum_{j \in \mathcal{N}^{-}(i)} \hspace*{-2mm} f_{(j,i)} = v_{i} \mbox{ for each } i\!\in\!\nodes.
\end{equation} 
\end{definition}
For a given empirical graph $\graph=(\nodes,\edges,\mW)$, we will consider flows that 
satisfy the capacity constraints: 
\begin{equation}
\label{equ_def_cap_constraints}
|f_{e}| \leq W_{e}
\end{equation} 
for some edges $e \in \edges$. Thus, we interpret the weights $W_{e}$ of the empirical graph 
as capacities of a flow network. At a later point, we will make explicit those edges for which the capacity 
constraints \eqref{equ_def_cap_constraints} has be satsified. 

We can associate any dual vector $\vy \in \mathbb{R}^{\nredges}$ with a particular flow $f^{(y)}$ whose 
values are given by $f^{(y)}_{e} \defeq W_{e} y_{e} $. It is then easy to verify that the flow $f^{(y)}$ satisfies 
the capacity constraints \eqref{equ_def_cap_constraints} and the conservation law \eqref{equ_conservation_law} with supplies 
$v_{i}$ if and only if 
\begin{equation} 
\label{equ_cond_flow_vector}
\hspace*{-3mm}\| \vy \|_{\infty}\!\leq\!1  \mbox{, } \mathbf{D}^{T} \vy\!=\!\mathbf{v} \mbox{ with }\mathbf{v}\!=\!(v_{1},\ldots,v_{\signalsize})^{T}\!\in\!\mathbb{R}^{\signalsize}. 
\end{equation} 
Thus, the magnitude $|y_{e}|$ of a dual vector entry represents the 
fraction of the edge capacity $W_{e}$ flowing through edge $e\!\in\!\edges$. 
\begin{proposition}
\label{prop_dual_TV_min_flow}
The dual problem \eqref{equ_dual_SLP} of TV minimization \eqref{equ_min_constr} is equivalent to the network optimizaton 
problem 
\begin{equation}
\label{equ_network_flow}
\max_{f \in \mathcal{R}} \sum_{i \in \trainingset}  x_{i} \sum_{j \in \mathcal{N}(i)} f_{(i,j)}, 
\end{equation}
with the constraint set $\mathcal{R}$ consisting of all flows that conform with \eqref{equ_def_cap_constraints} 
and \eqref{equ_conservation_law} with supplies $v_{i}$ satisfying 
\begin{equation}
\label{equ_cond_demand}
v_{i} = 0 \mbox{ for all unlabeled nodes } i\!\in\!\nodes\!\setminus\!\trainingset.
\end{equation}
In particular, $\vy$ solves \eqref{equ_dual_SLP} if and only if the flow $f^{(y)}$, defined 
edge-wise by $f^{(y)}_{e} = W_{e} y_{e}$, solves \eqref{equ_network_flow}. 
\end{proposition} 
\begin{proof}
The (extended-value) functions \eqref{equ_conv_conj_h} and \eqref{equ_conv_conj_g}, which 
constitute the dual problem \eqref{equ_dual_SLP}, implicitly constrain the dual vector $\vy$ to 
satisfy \eqref{equ_cond_flow_vector} with supplies of the form \eqref{equ_cond_demand}. Thus, 
any optimal dual vector $\widehat{\vy}$ induces a flow $f^{(\hat{y})} \in \mathcal{R}$. For any 
$\vy\!\in\!\mathbb{R}^{\nredges}$ such that the flow $f^{(y)}$ belongs to $\mathcal{R}$, the 
objective functions in \eqref{equ_network_flow} and \eqref{equ_dual_SLP} coincide. 
\end{proof} 
The problem \eqref{equ_network_flow} is an instance of a minimum-cost flow problem discussed in 
\cite[Ch.\ 1]{BertsekasNetworkOpt}. Various methods for solving minimum-cost flow problems are 
presented in \cite{BertsekasNetworkOpt}. 

Combining Proposition \ref{prop_dual_TV_min_flow} with the primal-dual optimality condition \eqref{equ_two_coupled_conditions} 
provides a characterization of the solutions of TV minimization in terms of particular network flows. 
\begin{corollary}
\label{cor_flow_satur_constant}
Given networked data with empirical graph $\graph$ and labels $\{x_{i} \}_{i \in \trainingset}$, 
consider some flow $\hat{f}$ which solves the minimum-cost flow problem \eqref{equ_network_flow}. 
Let us denote the set of edges which are not saturated in $\hat{f}$ by 
\begin{equation} 
\nonumber 
\mathcal{U}  \defeq \{ \{i,j\} \in \edges: | \hat{f}_{e} | < W_{e} \}. 
\end{equation} 
Then, any solution $\hat{\vx}$ of \eqref{equ_min_constr} satisfies $\hat{x}_{i}  = \hat{x}_{j}$ for each $e=\{i,j\} \in \mathcal{U}$. 
Thus, given some optimal flow $\hat{f}$ (which solves \eqref{equ_network_flow}), any solution to TV minimization is 
constant along edges which are not statured by $\hat{f}$. 
\end{corollary}
\begin{proof}
For the optimal flow $\hat{f}$ define the dual vector $\widehat{\vy} = \hat{f}_{e}/W_{e}$. 
According to Proposition \ref{prop_dual_TV_min_flow}, $\widehat{\vy}$ is a solution to the dual problem \eqref{equ_dual_SLP}. 
For this particular (optimal) dual vector $\widehat{\vy}$, any solution $\hat{\vx}$ to TV minimization has to satisfy 
the optimality condition \eqref{equ_two_coupled_conditions}. Using the right-hand condition in \eqref{equ_two_coupled_conditions} 
and the properties of the sub-differential $\partial g^{*}(\vy)$ (see \eqref{equ_conv_conj_g} and \cite[Sec. 32]{RockafellarBook}) 
yields the statement.
\end{proof}
Note that, for a particular edge $e=\{i,j\} \in \edges$ in the empirical graph, 
once we find at least one optimal flow $\hat{f}$ such that $|\hat{f}_{e}| < W_{e}$ we are 
assured that every solution to TV minimization is constant along that edge $e$. However, 
to apply Corollary \ref{cor_flow_satur_constant} we need an efficient means to 
construct or characterize flows which are optimal in the sense of \eqref{equ_network_flow}. 
While there exist some well-known methods for solving minimum-cost flow problems (see \cite{BertsekasNetworkOpt}), 
we consider Corollary \ref{cor_flow_satur_constant} mainly useful for (partially) characterizing the solutions of TV minimization. 
In order to actually solve TV minimization we will apply a different method 
which starts directly from the optimality conditions \eqref{equ_two_coupled_conditions}. 

\section{A Primal-Dual Method}
\label{sec_spl_Alg}

The solutions $\primslp$ of \eqref{equ_min_constr_unconstr} are characterized by 
\cite{RockafellarBook}
\begin{equation}
\label{equ_zero_subgradient}
\mathbf{0} \in \partial f(\primslp). 
\end{equation} 
Proximal methods solve \eqref{equ_min_constr_unconstr} via fixed-point iterations of an operator 
$\mathcal{P}$ whose fixed-points are the solutions $\hat{\vx}$ of \eqref{equ_zero_subgradient}, 
\begin{equation}
\label{equ_zero_subgradient_equ}
\mathbf{0} \in \partial f(\hat{\vx}) \mbox{ if and only if } \hat{\vx} = \mathcal{P} \hat{\vx}. 
\end{equation} 
In general, the operator $\mathcal{P}$ is not unique, i.e., there are different choices for $\mathcal{P}$ such that 
\eqref{equ_zero_subgradient_equ} is valid. These choices result in different proximal algorithms \cite{ProximalMethods}. One useful choice for $\mathcal{P}$ 
in \eqref{equ_zero_subgradient_equ} is suggested by the characterization \eqref{equ_two_coupled_conditions} 
of solutions to the primal \eqref{equ_min_constr_unconstr} and dual \eqref{equ_dual_SLP} form of TV minimization \eqref{equ_min_constr}. 
The resulting method has been presented in \cite[Alg.\ 1]{ComplexitySLP2018}. 

Let us detail the derivation of \cite[Alg.\ 1]{ComplexitySLP2018} which is re-stated as Alg.\ \ref{alg_sparse_label_propagation_centralized} below. 
Rewrite the two coupled conditions \eqref{equ_two_coupled_conditions} as 
\begin{align}
\primslp - {\bm \Gamma} \mD^{T} \dualslp & \in \primslp+ {\bm \Gamma} \partial h(\primslp)  \nonumber \\ 
2 {\bm \Lambda}  \mD \primslp +\dualslp  & \in {\bm \Lambda} \partial g^{*}(\dualslp)+ {\bm \Lambda} \mD\primslp+\dualslp,\label{equ_manipulated_coupled_conditions} 
\end{align}
with the invertible diagonal matrices (cf.\ \eqref{equ_edge_set_support_weights} and \eqref{equ_def_neighborhood})
\begin{align} 
{\bf \Lambda} & \defeq (1/2) {\rm diag} \{ \lambda_{\{i,j\}} = 1/W_{i,j} \}_{\{i,j\} \in \edges} \in \mathbb{R}^{\nredges \times \nredges} \mbox{ and }  \nonumber \\ 
{\bf \Gamma} & \defeq (1/2)  {\rm diag} \{ \gamma_{i} = 1/d_{i} \}_{i = 1}^{\signalsize} \in \mathbb{R}^{\signalsize \times \signalsize}.\label{equ_def_scaling_matrices}
\end{align}
The particular choice \eqref{equ_def_scaling_matrices} ensures that \cite[Lemma 2]{PrecPockChambolle2011}
\begin{equation}
\| {\bf \Gamma}^{1/2} \mD^{T} {\bf \Lambda}^{1/2} \|_{2} < 1, \nonumber
\end{equation}
which, in turn, guarantees convergence of the iterative algorithm we propose for solving \eqref{equ_min_constr_unconstr}. 

Using the concept of resolvent operators \cite[Sec. 1.1.]{PrecPockChambolle2011}, we further develop the characterization \eqref{equ_manipulated_coupled_conditions} of solutions $\hat{\vx}$ 
to TV minimization \eqref{equ_min_constr}. To this end we define the resolvent operators for the (set-valued) operators ${\bf \Lambda}  \partial g^{*}(\vy)$ and 
${\bm \Gamma} \partial h(\vx)$ (see \eqref{equ_min_constr_unconstr}) as 
\begin{align}
(\mathbf{I}\!+\!{\bf \Lambda} \partial g^{*})^{-1} (\vy) & \!\defeq\! \argmin\limits_{\vz \in \edgesigs} g^{*}(\vz)\!+\!(1/2) \| \vy\!-\!\vz \|_{{\bm \Lambda}^{-1}}^{2} \nonumber \\ 
(\mathbf{I}\!+\!{\bm \Gamma} \partial h)^{-1} (\vx) & \!\defeq\! \argmin\limits_{\vz \in \graphsigs} h(\vz)\!+\!(1/2) \| \vx\!-\!\vz\|_{{\bm \Gamma}^{-1}}^{2}. \label{equ_iterations_number_112}
\end{align}
Applying  \cite[Prop. 23.2]{Bauschke:2017} and \cite[Prop. 16.44]{Bauschke:2017} to the optimality condition \eqref{equ_manipulated_coupled_conditions} 
yields the equivalent condition (for $\primslp$, $\dualslp$ to be primal and dual optimal) 
\begin{align}
\primslp &= (\mathbf{I}\!+\!{\bm \Gamma} \partial h)^{-1} (\primslp\!-\!{\bm \Gamma} \mD^{T} \dualslp) \nonumber \\ 
\dualslp\!-\!2(\mathbf{I}\!+\!{\bf \Lambda}  \partial g^{*})^{-1}   {\bf \Lambda}  \mD \primslp& = (\mathbf{I}\!+\!{\bf \Lambda}  \partial g^{*})^{-1}(\dualslp\!-\!{\bf \Lambda}\mD\primslp).\label{equ_condition_fix_point}
\end{align}  
The characterization \eqref{equ_condition_fix_point} of the solution $\primslp \in \graphsigs$ for the TV minimization problem 
\eqref{equ_min_constr} leads naturally to the following coupled fixed-point iterations for finding a solution $\primslp$ of \eqref{equ_min_constr}: 
\begin{align}
\label{equ_fixed_point_iterations}  
\hat{\vy}^{(k+1)} &\defeq (\mathbf{I} + {\bf \Lambda}  \partial g^{*})^{-1}(\hat{\vy}^{(k)} +  {\bf \Lambda}  \mD(2\hat{\vx}^{(k)}- \hat{\vx}^{(k-1)}))\nonumber \\  
\hat{\vx}^{(k+1)} &\defeq (\mathbf{I} + {\bm \Gamma} \partial h)^{-1} (\hat{\vx}^{(k)} - {\bm \Gamma} \mD^{T} \hat{\vy}^{(k+1)}).
\end{align}  
Here, we used the diagonal matrices defined in \eqref{equ_def_scaling_matrices} as well as the incidence matrix $\mD$ (see \eqref{equ_def_incidence_mtx}). 
The fixed-point iterations \eqref{equ_fixed_point_iterations} are obtained as a special case of the iterations 
\cite[Eq. (4)]{PrecPockChambolle2011} when choosing $\theta\!=\!1$ (using the notation in \cite{PrecPockChambolle2011}). 

We implement the updates in \eqref{equ_fixed_point_iterations} by using simple closed-form expressions 
for the resolvent operators \eqref{equ_iterations_number_112} (see \cite[Sec. 6.2.]{pock_chambolle} for more details): 
\begin{align}
\label{equ_close_form_prox}
\hspace*{-3mm} (\mathbf{I}\!+\!{\bf \Lambda} \partial g^{*})^{-1} (\vy) &\!=\!(\tilde{y}_{1},\ldots,\tilde{y}_{N})^{T},\tilde{y}_{i}\!=\!y_{i}/\max \{ |y_{i}|, 1 \}  \nonumber \\ 
\hspace*{-3mm} (\mathbf{I}\!+\!{\bm \Gamma} \partial h)^{-1} (\tilde{\vx}) &\!=\!(t_{1},\ldots,t_{\signalsize})^{T}, t_{i}\!=\!\begin{cases} x_{i} & \hspace*{-3mm} \mbox{ for } i\!\in\!\trainingset \\ \tilde{x}_{i} & \hspace*{-3mm} \mbox{ otherwise.} \end{cases}
\end{align} 
Inserting \eqref{equ_close_form_prox} into the updates \eqref{equ_fixed_point_iterations} yields Alg.\ 
\ref{alg_sparse_label_propagation_centralized} for solving TV minimization \eqref{equ_min_constr}. 
Note that Alg.\ \ref{alg_sparse_label_propagation_centralized} is a special case of \cite[Alg.\ 1]{pock_chambolle} 
which uses a more general version of step $2$ in Alg.\ \ref{alg_sparse_label_propagation_centralized} of the 
form $\tilde{\vx}  \defeq \hat{\vx}^{(k)} +\theta( \hat{\vx}^{(k)} - \hat{\vx}^{(k\!-\!1)})$. Thus, step $2$ in Alg.\ 
\ref{alg_sparse_label_propagation_centralized} is obtained for the particular choice $\theta=1$. 
This choice ensures convergence of Alg.\ \ref{alg_sparse_label_propagation_centralized} with an optimal 
(worst-case) converge rate (see \cite{ComplexitySLP2018}). The tuning of $\theta$ is beyond the cope of this paper. 
Another difference between Alg.\ \ref{alg_sparse_label_propagation_centralized} and \cite[Alg.\ 1]{pock_chambolle} 
is the explicit computation of the running average in step $8$ (which is required for the convergence analysis 
underlying Proposition \ref{thm_conv_result}). 


We emphasize that Alg.\ \ref{alg_sparse_label_propagation_centralized} does not require knowledge of the 
partition $\partition$ underlying signal model \eqref{equ_def_clustered_signal_model}. It also does not involve any tuning parameters. 


\begin{algorithm}[htbp]
\caption{Primal-Dual Method for TV Minimization}{}
\begin{algorithmic}[1]
\renewcommand{\algorithmicrequire}{\textbf{Input:}}
\renewcommand{\algorithmicensure}{\textbf{Output:}}
\Require  empirical graph $\graph$ with incidence matrix $\mD\!\in\! \mathbb{R}^{\nredges \times \nrnodes}$ 
(see \eqref{equ_def_incidence_mtx}), training set $\trainingset$ with labels $\{ x_{i} \}_{i \in \trainingset}$. 
\Statex\hspace{-6mm}{\bf Initialize:} $k\!\defeq\!0$, $\bar{\vx} = \hat{\vx}^{(-1)}=\hat{\vx}^{(0)}=\hat{\vy}^{(0)}\!\defeq\! \mathbf{0}$, 
$\gamma_{i}\!\defeq\!1/d_{i}$, $\lambda_{\{i,j\}}\!=\!1/(2W_{i,j})$. 
\Repeat
\vspace*{2mm}
\State  $\tilde{\vx}  \defeq 2 \hat{\vx}^{(k)} - \hat{\vx}^{(k\!-\!1)}$
\vspace*{2mm}
\State $\hat{\vy}^{(k\!+\!1)}  \defeq \hat{\vy}^{(k)} + {\bf \Lambda}  \mD  \tilde{\vx}$ with ${\bf \Lambda}={\rm diag} \{ \lambda_{\{i,j\}} \}_{\{i,j\} \in \edges}$
\vspace*{2mm}
\State $\hat{y}_{e}^{(k\!+\!1)}\!\defeq\!\hat{y}_{e}^{(k\!+\!1)}\!/\!\max\{1, |\hat{y}_{e}^{(k\!+\!1)}| \}$  for every edge  $e\!\in\!\edges$
\vspace*{2mm}
\State $\hat{\vx}^{(k\!+\!1)}  \defeq \hat{\vx}^{(k)} - {\bm \Gamma} \mD^{T} \hat{\vy}^{(k\!+\!1)}$ with ${\bm \Gamma}={\rm diag} \{ \gamma_{i} \}_{i \in \nodes}$
\vspace*{2mm}
\State $\hat{x}_{i}^{(k\!+\!1)} \defeq x_{i}$  for every labeled node $i\!\in\!\trainingset$
\vspace*{2mm}
\State $k \defeq k\!+\!1$ 
\vspace*{2mm}
\State $\bar{\vx}^{(k)} \defeq (1\!-\!1/k) \bar{\vx}^{(k\!-\!1)}\!+\!(1/k) \hat{\vx}^{(k)} $
\vspace*{2mm}
\Until{stopping criterion is satisfied}
\vspace*{2mm}
\Ensure labels $\hat{x}_{i} \defeq \bar{x}_{i}^{(k)}$ \mbox{ for all nodes }$i \in \nodes$
\end{algorithmic}
\label{alg_sparse_label_propagation_centralized}
\end{algorithm}

There are various possible stopping criteria in Alg.\ \ref{alg_sparse_label_propagation_centralized}, including 
using a fixed number of iterations or testing for sufficient decrease of the objective function (see \cite{becker2011nesta} 
and Section \ref{experimental_results}). For testing if the objective function is decreased sufficiently, 
we can use the duality bound \eqref{equ_upper_bound_subopt} on the sub-optimality of the current 
objective function value $\| \bar{\vx}^{(k)} \|_{\rm TV}$. When using a fixed number of iterations, the 
following characterization of the convergence rate of Alg.\ \ref{alg_sparse_label_propagation_centralized} is 
helpful.  
\begin{proposition}[\hspace*{-1mm}\cite{ComplexitySLP2018}]
\label{thm_conv_result}
Consider the sequences $\hat{\vx}^{(k)}$ and $\hat{\vy}^{(k)}$ obtained from the update rule \eqref{equ_fixed_point_iterations} 
and starting from some initalizations $\hat{\vx}^{(0)}$ and $\hat{\vy}^{(0)}$. The averages 
\begin{equation} 
\label{equ_def_averages}
\bar{\vx}^{(\maxiter)}= (1/\maxiter) \sum_{k=1}^{\maxiter} \hat{\vx}^{(k)} \mbox{, and } \bar{\vy}^{(\maxiter)}= (1/\maxiter)\sum_{k=1}^{\maxiter} \hat{\vy}^{(k)}
\end{equation}
obtained after $\maxiter$ iterations 
of \eqref{equ_fixed_point_iterations}, satisfy 
\begin{align} 
\label{equ_bound_convergence_SLP}
\| \bar{\vx}^{(\maxiter)} \|_{\rm TV}\!-\! \| \primslp \|_{\rm TV} & \!\leq\!  \nonumber \\[2mm] 
&\hspace*{-20mm} (1/(2\maxiter)) \big( \| \hat{\vx}^{(0)}\!-\!\primslp \|^{2}_{{\bm \Gamma}^{-1}} + \| \hat{\vy}^{(0)}\!-\!\tilde{\vy}^{(\maxiter)} \|^{2}_{{\bm \Lambda}^{-1}}  \big)  
\end{align}
with $\tilde{\vy}^{(\maxiter)} = {\rm sign } \{\mD \bar{\vx}^{(\maxiter)}  \}$. 
Moreover, the sequence $\| \hat{\vy}^{(0)} - \tilde{\vy}^{(\maxiter)} \|_{{\bm \Lambda}^{-1}}$, for $\maxiter=1,2,\ldots$ is bounded. 
\end{proposition} 
According to \eqref{equ_bound_convergence_SLP}, the sub-optimality of Alg.\ \ref{alg_sparse_label_propagation_centralized} 
after $K$ iterations is bounded as 
\begin{equation}
\label{equ_upper_bound_constant} 
\| \bar{\vx}^{(\maxiter)} \|_{\rm TV} - \| \primslp \|_{\rm TV}  \leq c/K,
\end{equation} 
where the constant $c$ does not depend on $K$ but might depend on the empirical graph $\graph$, via its weighted incidence matrix 
$\mD$ \eqref{equ_def_incidence_mtx}, as well as on the initial labels $\{ x_{i} \}_{i \in \trainingset}$. 
The bound \eqref{equ_upper_bound_constant} suggests that in order to ensure reducing the sub-optimality by a factor of two, we need 
to run Alg.\ \ref{alg_sparse_label_propagation_centralized} for twice as many iterations. The upper bound \eqref{equ_upper_bound_constant} 
is tight among all message passing (local) methods for solving \eqref{equ_min_constr}. 
In particular, the rate $1/K$ cannot be improved for a chain-structured empirical graph (see \cite{ComplexitySLP2018}). 

As indicated by \cite[Thm. 3.2]{Condat2013},  Alg.\ \ref{alg_sparse_label_propagation_centralized} is robust to numerical 
errors arising during the updates, which can be a crucial property for high-dimensional problems. 


The computational cost of one iteration in Alg.\ \ref{alg_sparse_label_propagation_centralized} is proportional to the 
number of edges in the empirical graph $\graph$. This can be verified by noting that Alg.\ \ref{alg_sparse_label_propagation_centralized} 
can be implemented as message passing on the empirical graph (see Alg.\ \ref{sparse_label_propagation_mp}). 
Thus, for a fixed number $K$ of iterations, the computational cost of Alg.\ \ref{alg_sparse_label_propagation_centralized} 
is proportional to the number of edges in the empirical graph. In contrast, the computational cost of 
state-of-the art maximum flow algorithms can be considerably higher \cite{GoldbergTarjan2014,Orlin2013}. 
Moreover, while Alg.\ \ref{alg_sparse_label_propagation_centralized} allows for a rather straightforward implementation 
on modern big data computing frameworks (see Section \ref{sec_LFR_graph}), this is typically more challenging for maximum flow 
methods which are (partially) based on combinatorial search (see \cite[Sec. 3.3.]{PrecPockChambolle2011}). 

We now show how to obtain a scalable implementation of Alg.\ \ref{alg_sparse_label_propagation_centralized} 
using message passing over the underlying empirical graph $\graph$. This message passing formulation, 
summarized in Alg.\ \ref{sparse_label_propagation_mp} (being a slight reformulation of \cite[Alg.\ 2]{ComplexitySLP2018}), 
is obtained by implementing the application of the graph 
incidence matrix $\mD$ and its transpose $\mD^{T}$ (cf. steps $2$ and $5$ of Alg.\ \ref{alg_sparse_label_propagation_centralized}) by local updates of 
the labels $\hat{x}_{i}$, i.e., updates which involve only the neighbourhoods $\mathcal{N}(i)$, $\mathcal{N}(j)$ of all edges $\{i,j\} \in \edges$ in the empirical graph $\graph$.  

Note that executing Alg.\ \ref{sparse_label_propagation_mp} does not require global knowledge (such as the maximum node degree $d_{\rm max}$ \eqref{equ_def_max_node_degree}) 
about the entire empirical graph. Indeed, if we associate each node in the data graph with a computational unit, execution of Alg.\ \ref{sparse_label_propagation_mp} requires each node 
$i \in \nodes$ only to store the neighboring values $\{ \hat{y}_{\{i,j\}}, W_{i,j} \}_{j \in \mathcal{N}(i)}$ and $\hat{x}_{i}^{(k)}$. Moreover, the number of arithmetic operations required at each node $i \in \nodes$ 
during each time step is proportional to the number $|\mathcal{N}(i)|$ of its neighbours $\mathcal{N}(i)$. 
Thus, Alg.\ \ref{sparse_label_propagation_mp} can be scaled to large datasets which can be represented as 
sparse networks having small maximum degree $d_{\rm max}$ \eqref{equ_def_max_node_degree}. The datasets generated 
in many important applications are accurately represented by such sparse networks \cite{barabasi2016network}. 

\begin{algorithm}[h]
\caption{Distributed Implementation of Alg.\ \ref{alg_sparse_label_propagation_centralized}}{}
\begin{algorithmic}[1]
\renewcommand{\algorithmicrequire}{\textbf{Input:}}
\renewcommand{\algorithmicensure}{\textbf{Output:}}
\Require empirical graph $\graph=(\nodes,\edges,\mW)$, training set $\trainingset$ with labels $\{ x_{i} \}_{i \in \trainingset}$. 
\vspace*{3mm}
\Statex\hspace{-6mm}{\bf Initialize:} $k\!\defeq\!0$, $\bar{\vx}=\hat{\vy}^{(0)}=\hat{\vx}^{(-1)}=\hat{\vx}^{(0)}\!\defeq\!\mathbf{0}$, 
$\gamma_{i}\!\defeq\!1/d_{i}$. 
\vspace*{1mm}
\Repeat
\vspace*{2mm}
\State for all nodes $i \in \nodes$: $\tilde{x}_{i}  \defeq 2 \hat{x}^{(k)}_{i}  - \hat{x}^{(k-1)}_{i}$    
\vspace*{2mm}
\State for all edges $e=(i,j)\!\in\!\edges$: 
\begin{equation}
\nonumber 
\hat{y}_{e}^{(k+1)}  \defeq \hat{y}^{(k)}_{e} +    (1/2)  (\tilde{x}_{e^{+}} - \tilde{x}_{e^{-}})
\end{equation}
\State for all edges $e \in \edges$: 
\begin{equation} 
\nonumber
\hat{y}_{e}^{(k+1)}  \defeq \hat{y}_{e}^{(k+1)} / \max\{1, |\hat{y}_{e}^{(k+1)}| \}
\end{equation} 
\State for all nodes $i\!\in\!\nodes$: 
\begin{equation}  
\nonumber
\hspace*{-3mm}\hat{x}^{(k+1)}_{i}\!\defeq\!\hat{x}^{(k)}_{i}\!-\!\gamma_{i} \bigg[ \hspace*{0mm}\sum\limits_{j \in \mathcal{N}^{+}(i)} \hspace*{-1mm}W_{i,j} \hat{y}^{(k+1)}_{(i,j)} \hspace*{-1mm}- \hspace*{-1mm}\sum\limits_{j \in \mathcal{N}^{-}(i)} \hspace*{-1mm}W_{i,j} \hat{y}^{(k+1)}_{(j,i)} \hspace*{0mm}\bigg]   
\end{equation}
\State for all labeled nodes $i\!\in\!\trainingset$:  $\hat{x}^{(k+1)}_{i} \defeq x_{i}$
\vspace*{2mm}
\State $k \defeq k+1$    
\vspace*{2mm}
\State for all nodes $i\!\in\!\nodes$: $\bar{x}_{i} \defeq (1-1/k)\bar{x}_{i} + (1/k) \hat{x}^{(k)}_{i}$
\vspace*{2mm}
\Until{stopping criterion is satisfied}
\vspace*{1mm}
\Ensure labels $\hat{x}_{i} \defeq \hat{x}^{(k)}_{i}$ \mbox{ for all }$i \in \nodes$
\end{algorithmic}
\label{sparse_label_propagation_mp}
\vspace*{-1mm}
\end{algorithm}
Alg.\ \ref{alg_sparse_label_propagation_centralized} implicitly also solves the dual problem \eqref{equ_dual_SLP} 
of TV minimization \eqref{equ_min_constr}. We might therefore interpret Alg.\ \ref{sparse_label_propagation_mp} 
as a message passing method for network optimization. In particular, associate the current approximation $\hat{\vy}^{(k)}$ 
for the optimal dual vector $\hat{\vy}$ (see \eqref{equ_dual_SLP}) with the flow $f^{(k)}: \edges\!\rightarrow\!\mathbb{R}$ 
having values $f^{(k)}_{e}\!\defeq\!W_{e} y^{(k)}_{e}$. Then, step $4$ of Alg.\ \ref{sparse_label_propagation_mp} aims at 
enforcing the capacity constraint \eqref{equ_def_cap_constraints} for the flow $f^{(k)}$. 
Moreover, step $5$ amounts to updating the current signal estimate $\hat{x}_{i}^{(k)}$, for each unlabeled node $i \in\!\nodes\!\setminus \trainingset$,  
by the (scaled) demand induced by the current flow $f^{(k)}$ \eqref{equ_conservation_law}. 
Thus, for each unlabeled node $i\!\in\!\nodes \setminus \trainingset$, we might interpret the signal estimates $\hat{x}_{i}^{(k)}$ 
as the (scaled) cumulative demand induced by the flows $f^{(k')}$ for $k'=1,\ldots,k$. The labeled nodes $i \in \trainingset$ have a constant supply 
$\hat{x}^{(k)}_{i} = x_{i}$ whose amount is the label $x_{i}$. Step $3$ of Alg.\ \ref{sparse_label_propagation_mp} balances 
discrepancies between accumulated demands $\hat{x}^{(k)}_{i}$ at the different nodes by adapting the flow 
$f^{(k)}_{(i,j)}$ through an edge $e=(i,j) \in \edges$ according to the difference $(\tilde{x}_{i} - \tilde{x}_{j})$.

\section{When is TV Minimization Accurate?} 
\label{sec_main_results} 

We now provide conditions which ensure that any solution $\hat{\vx}$ of TV minimization 
\eqref{equ_min_constr} is close to the true underlying graph signal $\vx\!=\!(x_{1},\ldots,x_{\signalsize})^{T}\!\in\!\mathbb{R}^{\signalsize}$ 
which can be well approximated by a piece-wise constant graph signal \eqref{equ_def_clustered_signal_model}. 

Since TV minimization \eqref{equ_min_constr} is a particular case of $\ell_{1}$ minimization \cite{KabRau2015Chap}, 
successful recovery is ensured by the stable analysis nullspace property (see \cite[Lemma 5]{NNSPFrontiers2018}). 

As we show in Proposition \ref{lem_flow_cond}, the stable analysis nullspace property is ensured if the nodes in the 
training set are sufficiently well connected to the cluster boundaries $\boundary$. To this end,  
we define the notion of resolving training sets. 
\begin{definition}
\label{def_sampling_set_resolves}
Consider a partition $\partition=  \{ \cluster_{1},\cluster_{2},\ldots,\cluster_{|\partition|} \}$ of the empirical graph $\graph=(\nodes,\edges,\mW)$ 
into disjoint subsets of nodes (clusters) $\cluster_{l} \subseteq \nodes$. 
A training set $\trainingset \subseteq \nodes$ resolves the partition $\partition$ if, for any collection of signs $\{b_{e} \in \{-1,1\}\}_{e \in \boundary}$, 
there exists a flow $f: \edges \rightarrow \mathbb{R}$ such that 
\begin{align} 
\label{equ_resolv_cluster}
& f_{(i,j)}  \!=\! b_{(i,j)} 2 W_{i,j} \mbox{ for each } (i,j)\!\in\!\boundary \nonumber \\[3mm]
& |f_{(i,j)}| \!\leq\!W_{i,j} \mbox{ for each } (i,j)\!\in\!\edges\!\setminus\!\boundary  \\[3mm]
&\hspace*{-1mm} \sum_{(i,j) \in \edges} \hspace*{-1mm} f_{(i,j)}\!-\hspace*{-2mm}\sum_{(j,i) \in \edges} \hspace*{-1mm}f_{(j,i)}\!=\!0 \mbox{ for each } i\!\in\!\nodes\!\setminus\!\trainingset.\nonumber
\end{align} 
\end{definition} 
We highlight that Definition \ref{def_sampling_set_resolves} is only required for the analysis of the solutions of 
TV minimization \eqref{equ_min_constr}. In order to use Alg.\ \ref{alg_sparse_label_propagation_centralized} 
for solving \eqref{equ_min_constr}, we do not need any to place any requirements on the training set $\trainingset$. 
We can perfectly use Alg.\ \ref{alg_sparse_label_propagation_centralized} also when the training set $\trainingset$ 
does not resolve the partition $\partition$ underlying the signal model \eqref{equ_def_clustered_signal_model}. 
However, in this case we cannot guarantee that the estimate delivered by Alg.\ \ref{alg_sparse_label_propagation_centralized} is 
close to the true underlying graph signal. 

It is important to note that Definition \ref{def_sampling_set_resolves} involves both the labeled 
training set $\trainingset$ and the partition $\partition$. For a given training set $\trainingset$, 
we can increase the chance of satisfying \eqref{equ_resolv_cluster} by optimizing the partition $\partition$ 
underlying \eqref{equ_def_clustered_signal_model}. 
Enlarging the training set $\trainingset$ (by acquiring more labels), will  
increase the chance of satisfying \eqref{equ_resolv_cluster} as there are fewer unlabeled 
nodes for which the last condition in \eqref{equ_resolv_cluster} has to be ensured. 

Definition \ref{def_sampling_set_resolves} requires a sufficiently large network flow (across cluster boundaries) 
between the labeled nodes $\trainingset$. These network flows have to be such that the boundary edges 
$e \in \partial \partition$ are flooded (or saturated) with an amount of flow at least $2 W_{e}$. The training set 
$\trainingset \subseteq \nodes$ depicted in Fig.\ \ref{fig_clustered_graph_signal} resolves the partition $\partition=\{\cluster_{1},\cluster_{2}\}$.  

\begin{proposition}[Thm. 4 in \cite{NNSPFrontiers2018}] 
\label{main_thm_approx_sparse}
Consider data with empirical graph $\graph$ and true labels $x_{i}$ forming a graph signal $\vx \in \graphsigs$.  
We are provided with observed labels $x_{i}$ at nodes in the training set $\trainingset$. 
If $\trainingset$ resolves the partition $\partition=\{\cluster_{1},\ldots,\cluster_{|\partition|}\}$, 
any solution $\hat{\vx}$ of \eqref{equ_min_constr} satisfies
\begin{equation}
\label{equ_bound_TVmin_TV}
\| \hat{\vx}\!-\!\vx \|_{\rm TV} \leq 6  \min_{\{ a_{l} \}_{l=1}^{|\partition|}} \big\| \xsig- \sum_{l=1}^{|\partition|} a_{l} \mathcal{I}_{\cluster_{l}}[\cdot] \big\|_{\rm TV}, 
\end{equation} 
\end{proposition} 
For convenience, we spell out a bound on the error $\hat{x}_{i}\!-\!x_{i}$ itself which is a direct consequence of \eqref{equ_bound_TVmin_TV}.
\begin{corollary}
Under the same assumptions as in Proposition \ref{main_thm_approx_sparse}, any solution of \eqref{equ_min_constr} satisfies
\begin{equation}
\label{equ_bound_TVmin_TV_ell1}
\max_{i \in \nodes} | \hat{x}_{i}\!-\!x_{i}| \leq 6 d_{\rm max}  \min_{\{ a_{l} \}_{l=1}^{|\partition|}} \big\| \xsig- \sum_{l=1}^{|\partition|} a_{l} \mathcal{I}_{\cluster_{l}}[\cdot] \big\|_{1}.
\end{equation} 
\end{corollary} 
\begin{proof}
The bound \eqref{equ_bound_TVmin_TV_ell1} is obtained from \eqref{equ_bound_TVmin_TV} using 
the inequality $\| \vz \|_{\rm TV}\!\leq\!d_{\rm max} \| \vz \|_{1}$ (see \eqref{equ_def_TV}) with the maximum weighted degree $d_{\rm max}$ \eqref{equ_def_max_node_degree}. 
\end{proof} 
Thus, if the training set $\trainingset$ resolves the partition underlying \eqref{equ_def_clustered_signal_model}, 
any solution $\hat{\vx}$ to TV minimization \eqref{equ_min_constr} is close (in TV seminorm) to the true 
labels if they can be well approximated by a piece-wise constant graph signal \eqref{equ_def_clustered_signal_model}. 
For labels forming exactly a piece-wise constant signal, we can specialize Proposition \ref{main_thm_approx_sparse} as follows. 
\begin{corollary}[Thm. 3 in \cite{NNSPFrontiers2018}]
\label{corr_exact_recovery}
Consider data with empirical graph $\graph$ and true labels $x_{i}$ forming a piece-wise constant graph signal 
$\vx \in \graphsigs$ (see \eqref{equ_def_clustered_signal_model}) over the partition $\partition=\{\cluster_{1},\ldots,\cluster_{|\partition|}\}$. 
If the training set $\trainingset$ resolves $\partition$, the solution $\hat{\vx}$ of \eqref{equ_min_constr} is unique and coincides with $\vx$. 
\end{corollary} 
 
We emphasize that Alg.\ \ref{alg_sparse_label_propagation_centralized} does not require knowledge of the 
partition $\partition=\{\mathcal{C}_{1},\ldots,\mathcal{C}_{|\partition|}\}$. Indeed, we could use Alg.\ 
\ref{alg_sparse_label_propagation_centralized} to determine the clusters $\cluster_{l}$ if the underlying 
labels $x_{i}$ form a piece-wise constant signal $x_{i} = \sum_{l=1}^{|\partition|} a_{l} \mathcal{I}_{\cluster_{l}}[i]$ 
with $a_{l} \neq a_{l'}$ for different clusters $l \neq l'$.

Proposition \ref{main_thm_approx_sparse} and Corollary \ref{corr_exact_recovery} require the  
partition $\partition$ in \eqref{equ_def_clustered_signal_model} to be resolved by the training set $\trainingset$. 
The direct verification if a given partition is resolved by a particular training set is computationally challenging 
as it involves an exponential number of constraints \eqref{equ_resolv_cluster} to be evaluated. However, 
if the empirical graph is modeled using a probabilistic model, such as the stochastic block model (SBM) \cite{AbbeSBM2018}, 
we can make use of large deviation results to determine network parameter regimes such that 
\eqref{equ_resolv_cluster} is satisfied with high probability \cite{Karger1999}. 

We now show how to verify the validity of \eqref{equ_resolv_cluster} 
using maximum flow algorithms \cite{BertsekasNetworkOpt,KleinbergTardos2006}. To this end, 
we define a particular subgraph $\graph_{l}$ associated with the clusters $\cluster_{l}$ of a 
partition $\partition=\{\cluster_{1},\ldots,\cluster_{|\partition|}\}$ which is resolved by $\trainingset$. 
\begin{definition} 
\label{equ_def_augmented_cluster_graph}
For a given cluster $\cluster_{l} \subseteq \nodes$ within the empirical graph $\graph=(\nodes,\edges,\mathbf{W})$, we define the augmented cluster subgraph 
$\graph_{l}\!=\!(\cluster_{l}\!\cup\!\{0\},\edges_{l},\mathbf{C}^{(l)})$ whose nodes are constituted by the cluster $\cluster_{l}$ and the additional node $0$. 
The edge set $\edges_{l}$ of $\graph_{l}$ is defined as  
\begin{equation}
\edges_{l} =  \{ \{i,j\}\!\in\!\edges: i,j\!\in\!\cluster_{l} \}\!\cup\!\{  \{0,i\}: i \in \partial \cluster_{l}\!\cap\!\cluster_{l} \}.   
\end{equation}  
Thus, the edges $\edges_{l}$ of the augmented cluster subgraph $\graph_{l}$ are constituted by (i) the intra-cluster edges $\{ \{i,j\}\!\in\!\edges: i,j\!\in\!\cluster_{l} \}$ connecting 
nodes within cluster $\cluster_{l}$ of the empirical graph $\graph$ and (ii) one additional edge $\{0,i\}$ for each node $i\!\in\!\partial \cluster_{l}\!\cap\!\cluster_{l}$ on the boundary 
of cluster $\cluster_{l}$. 
The weights $C^{(l)}_{e}$ of the edges $e \in \edges_{l}$ in the graph $\graph_{l}$ are defined as 
\begin{equation}
\label{equ_capacity_intra_cluster}
\hspace*{-1mm}C^{(l)}_{e}\!\defeq\! W_{i,j}  \mbox{ for every edge } e\!=\!\{i,j\}\!\in\!\edges \mbox{ with } i,j\!\in\!\cluster_{l} 
\end{equation} 
and 
\begin{equation}
\label{equ_boundary_capacity}
C^{(l)}_{\{0,i\}}\!\defeq\!2 \sum_{j \in \mathcal{N}(i)\!\setminus\!\cluster_{l}}  W_{i,j}  \mbox{ for each node } i\!\in\!\partial \cluster_{l}\!\cap\!\cluster_{l}.
\end{equation} 
\end{definition} 
To illustrate Definition \ref{equ_def_augmented_cluster_graph}, Fig.\ \ref{fig_augmented_subgraphs} depicts 
the augmented subgraphs of the clusters in the empirical graph in Fig.\ \ref{fig_clustered_graph_signal}. 

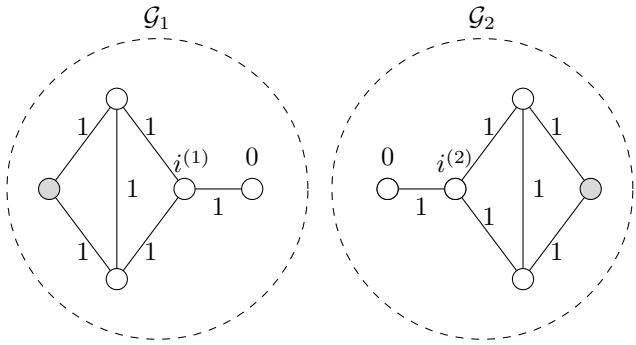
\begin{figure}[htbp]
\begin{center}
\hspace*{-5mm}
\begin{tikzpicture}
    \tikzset{x=0.9cm,y=1.2cm,every path/.style={>=latex},node style/.style={circle,draw}} 
    \coordinate[] (x2) at (0,0); 
    \coordinate[] (x10) at (-1,1);    
    \coordinate[] (x3) at (0,2);
    \coordinate[] (x4) at (1,1);  
    \draw[line width=0.4,-] (x2) edge  node [right] {$1$} (x3);
    \draw[line width=0.4,-] (x2) edge  node [below] {$1$} (x10);
    \draw[line width=0.4,-] (x3) edge   node [above] {$1$} (x10);
    \draw[line width=0.4,-] (x2) edge   node [below] {$1$} (x4);
    \draw[line width=0.4,-] (x3) edge   node [above] {$1$}(x4);
    \coordinate[] (x7) at (2,1);
    \draw[line width=0.4,-] (x4) edge node [below] {$1$}    (x7);
    \coordinate[] (x13) at (4,1) ;    
    \coordinate[] (x12) at (5,1) ;    
    \coordinate[] (x5) at (6,0); 
    \coordinate[] (x6) at (6,2);    
    \coordinate[] (x11) at (7,1);              
        \draw[line width=0.4,-] (x12) edge node [below] {$1$}    (x13);
   \draw[line width=0.4,-] (x5) edge node[above]{$1$} (x12);
      \draw[line width=0.4,-] (x6) edge node[above]{$1$} (x12); 
   \draw[line width=0.4,-] (x5) edge node[right]{$1$} (x6);
   \draw[line width=0.4,-] (x11) edge node[above]{$1$}(x6);
            \draw[line width=0.4,-] (x5) edge  node [below] {$1$} (x11);
    \draw [fill=white]   (x2)  circle (4pt);
    \draw [fill=white]   (x3)  circle (4pt);
    \draw [fill=gray!30]   (x10)  circle (4pt);
    \draw [fill=white]   (x4)  circle (4pt) node [above=1mm] {\hspace*{2mm}$i^{(1)}$} ;
   \draw [fill=white]   (x5)  circle (4pt);
    \draw [fill=white]   (x6)  circle (4pt);
       \draw [fill=white]   (x12)  circle (4pt);
              \draw [fill=white]   (x13)  circle (4pt);
    \draw [fill=gray!30]   (x11)  circle (4pt);
    \draw [fill=white]   (x13)  circle (4pt) node [above=2mm] {$0$} ;
        \draw [fill=white]   (x12)  circle (4pt) node [above=1mm] {$i^{(2)}$} ;
    \draw [fill=white]   (x7)  circle (4pt) node [above=2mm] {$0$} ;
    \node[draw,circle,dashed,minimum size=4cm,inner sep=0pt,label={$\mathcal{G}_1$}] at (0.6,1) {};
        \node[draw,circle,dashed,minimum size=4cm,inner sep=0pt,label={$\mathcal{G}_2$}] at (5.4,1) {};
\end{tikzpicture}
\end{center}
\caption{Augmented subgraphs $\graph_{1}, \graph_{2}$ obtained from the partitioned empirical graph in Fig.\ \ref{fig_clustered_graph_signal}. 
Each subgraph is obtained from a cluster $\cluster_{l}$ by adding edges from each boundary node $i^{(l)}\!\in\!\partial\cluster_{l}$ to the augmented node $0$. 
The numbers indicate the capacity constraints \eqref{equ_def_cap_constraints} along the edges.}
\label{fig_augmented_subgraphs} 
\end{figure} 

\begin{proposition}
\label{lem_flow_cond}
Consider an empirical graph $\graph = (\nodes,\edges,\mathbf{W})$ which is partitioned into the 
clusters $\partition = \{\cluster_{1},\ldots,\cluster_{|\partition|} \}$. Assume that each cluster $\cluster_{l}$ 
contains at least one labeled node $i^{(l)} \in \cluster_{l} \cap \trainingset$ from the training set $\trainingset \subseteq \nodes$. 
If, for each cluster $\cluster_{l}$, the corresponding subgraph $\graph_{l}$ (see Definition \ref{equ_def_augmented_cluster_graph}) 
supports a network flow (using the capacities \eqref{equ_capacity_intra_cluster} and \eqref{equ_boundary_capacity} 
for the capacity constraints \eqref{equ_def_cap_constraints}) of value $2 \sum_{e \in \partial \cluster_{l}} W_{e}$ between 
the source node $i^{(l)}$ and the sink node $0$, then the training set $\trainingset$ resolves the partition $\partition$. 
\end{proposition} 
\begin{proof} 
Consider a particular cluster $\cluster_{l}$ containing the labeled node $i^{(l)}  \in \cluster_{l} \cap \trainingset$. 
By assumption, the associated subgraph $\graph_{l}$ supports a network flow between $i^{(l)}$ and the extra 
node $0$ of value $2 \sum_{e \in \partial \cluster_{l}} W_{e}$. The max-flow/min-cut theorem (see \cite[Thm. 6.1.6]{JungnckelBook}) 
implies that this flow value can only be achieved if, for each subset $\mathcal{A} \subseteq \cluster_{l} \setminus \{ i^{(l)} \}$,  
the total capacity of the edges $\{ \{i,j\}\!\in\!\edges: i\!\in\!\mathcal{A}, j\!\in\!\cluster_{l}\!\setminus\!\mathcal{A} \}$ 
is at least as large as twice the total capacity of the edges $\{ \{i,j\} \in \edges: i \in \mathcal{A}, j\!\in\!\nodes\!\setminus\!\cluster_{l} \}$, 
\begin{equation}
\label{equ_condition_minimum_flow} 
\sum_{ \{i,j\} \in \edges: i \in \mathcal{A}, j \in \cluster_{l} \setminus \mathcal{A} } \hspace*{-3mm} W_{i,j}\!\geq\!2 \sum_{ \{i,j\} \in \edges: i \in \mathcal{A}, j \in \nodes \setminus \cluster_{l}}  \hspace*{-3mm} W_{i,j}. 
\end{equation} 
The validity of \eqref{equ_condition_minimum_flow}, for each cluster $\cluster_{l}$ of the partition $\partition$, 
implies via Hoffman's circulation theorem \cite[Thm. 10.2.7]{JungnckelBook} the existence of a network flow 
satisfying the requirements \eqref{equ_resolv_cluster} for the training set $\trainingset$ to resolve the partition $\partition$. 
\end{proof} 
In Section \ref{sec_two_cluster}, we will demonstrate the usefulness of Proposition \ref{lem_flow_cond} for certifying the accuracy 
of Alg.\ \ref{alg_sparse_label_propagation_centralized}. Moreover, we can combine Proposition \ref{lem_flow_cond} with existing 
results from graph sampling to characterize TV minimization for empirical graphs that can be well approximated by an SBM. In particular, 
\cite[Theorem 2.1]{Karger1999} allows us to verify if the conditions of Proposition \ref{lem_flow_cond} are satisfied (with high probability) 
based on the expected values of cuts in the graph $\graph$. According to Proposition \ref{lem_flow_cond}, TV minimization is accurate 
if there exists a flow from the labeled nodes $\trainingset \cap \cluster_{l}$ in each cluster to its boundary $\partial \cluster_{l}$ of value 
$2 \sum_{e \in \partial \cluster_{l}} W_{e}$. A simple argument based on \cite[Theorem 2.1]{Karger1999} shows that this condition is 
satisfied with high probability for an SBM (with cluster sizes not too small), whenever 
\begin{equation}
\label{equ_condition_SBM}
| \trainingset \cap \cluster_{l} |  p_{\rm in} \gg 2 p_{\rm out}  (|\nodes| - |\cluster_{l}|). 
\end{equation} 
Here, $p_{\rm in}$ ($p_{\rm out}$) denotes the probability that two nodes from the same cluster 
(from different clusters) are connected by an edge. Condition \eqref{equ_condition_SBM} allows 
to characterize parameter regimes for the SBM such that TV minimization can recover piece-wise 
constant signals from a given number of labeled nodes. We will verify condition \eqref{equ_condition_SBM} empirically 
in Section \ref{sec_SBM_graph}. 

Proposition \ref{main_thm_approx_sparse} and Corollary \ref{corr_exact_recovery} requires each cluster $\cluster_{l}$ in \eqref{equ_def_clustered_signal_model} 
to contain at least one labeled node $i\!\in\!\trainingset$ (see Definition \ref{def_sampling_set_resolves}). 
However, even if this condition is not met we still can say something about the solutions 
of TV minimization \eqref{equ_min_constr}. In particular, the optimality condition \eqref{equ_two_coupled_conditions} 
requires any solution $\hat{\vx}$ of TV minimization \eqref{equ_min_constr} to be constant around labeled nodes $i \in \trainingset$. 
The graph signal $\hat{\vx}$ can only change along edges $e=\{i,j\} \in \edges$ which are saturated, i.e., $|\hat{y}_{e}| = 1$ 
holds for every dual solution $\hat{\vy}$ of \eqref{equ_dual_SLP} (see Corollary \ref{cor_flow_satur_constant}).

\section{Numerical Experiments}
\label{experimental_results}
 
We assess the statistical and computational performance of Alg.\ \ref{alg_sparse_label_propagation_centralized} 
using numerical experiments involving synthetic and ``real-world'' data. 
The first experiment discussed in Section \ref{sec_two_cluster} revolves around an ensemble of synthetic datasets 
whose empirical graphs consist of two clusters with varying level of connectivity. We verify the recovery condition provided by 
Proposition \ref{main_thm_approx_sparse} by computing the recovery error of Alg. \ref{alg_sparse_label_propagation_centralized} 
as the cluster connectivity varies. Section \ref{sec_SBM_graph} discusses the application of TV minimization to a synthetic empirical graph 
generated using an SBM. 
In Section \ref{sec_LFR_graph}, we verify the scalability of Alg.\ \ref{alg_sparse_label_propagation_centralized} by 
implementing its message passing formulation Alg.\ \ref{sparse_label_propagation_mp} in a big data framework. 
Finally, in Section \ref{sec_road_network}, we discuss the application of Alg.\ \ref{alg_sparse_label_propagation_centralized} 
to data obtained from a Danish road network. 

To allow for reproducible research, we have made the source code for the numerical experiments discussed in 
Section \ref{sec_two_cluster} and Section \ref{sec_SBM_graph} available at \url{https://github.com/alexjungaalto/ResearchPublic/tree/master/TVMin}. 
The source code for the numerical experiments discussed in Section \ref{sec_LFR_graph} and Section \ref{sec_road_network} 
can be found  at \url{https://github.com/Dru-Mara/GraphSignalRecovery}. 

\subsection{Two-Cluster Graph} 
\label{sec_two_cluster}

In this experiment, we generate an empirical graph $\graph$ by first generating two clusters 
$\cluster_{1}$ and $\cluster_{2}$ of size $N/2 = 100$ drawn from an Erd{\"o}s-Renyi ensemble with 
varying edge occurrence probability. We then connected those two clusters by randomly placing 
edges between them. The resulting empirical graph $\graph$ is then assigned a piece-wise constant 
graph signal $\vx$ of the form \eqref{equ_def_clustered_signal_model} using the partition 
$\partition=\{\cluster_{1},\cluster_{2} \}$. We apply Alg.\ \ref{alg_sparse_label_propagation_centralized} 
to recover the graph signal $\vx$ based only on its values at the nodes in the training set $\trainingset$ 
which contains exactly one node from each of the two clusters, i.e., $|\trainingset|\!=\!2$. 

Using Proposition \ref{lem_flow_cond}, we can verify if the partition $\partition=\{\cluster_{1}, \cluster_{2}\}$ is resolved by the 
training set $\trainingset$ by computing, for each cluster $\cluster_{l}$ the network flow between the 
labeled node $i \in \cluster_{l} \cap \trainingset$ and the boundary $\partial \cluster_l$. Let $\rho^{(l)}$ 
denote the resulting flow value, normalized by the total weight of the boundary $2 \sum\limits_{e \in \partial \cluster_{l}} W_{e}$. 
According to Proposition \ref{lem_flow_cond}, the partition $\partition$ is resolved by $\trainingset$ if $\rho^{(l)} \geq 2$ for all $l=1,2$. 

In Fig.\ \ref{fig_NMSEconnect}, we depict the normalized mean squared error (NMSE) 
$\varepsilon \defeq \| \vx\!-\!\tilde{\vx}^{(k)} \|^{2}_{2} / \| \vx^{(k)} \|^{2}_{2}$ incurred by 
Alg.\ \ref{alg_sparse_label_propagation_centralized} (averaged over $10$ i.i.d.\ simulation runs) 
for varying connectivity, as measured by the empirical average $\bar{\rho}$ of $\rho^{(1)}$ and $\rho^{(2)}$ 
(which have the same distribution due to the symmetric graph construction). The results in Fig.\ \ref{fig_NMSEconnect} 
agrees with our analysis (see Proposition \ref{lem_flow_cond} and Proposition \ref{main_thm_approx_sparse}) 
which predicts that TV minimization Alg.\ \ref{alg_sparse_label_propagation_centralized} is accurate (incurring small NMSE) 
if the cluster $\cluster_{1}$ and $\cluster_{2}$ are well connected such that $\rho^{(1)}, \rho^{(2)} \geq 2$. 

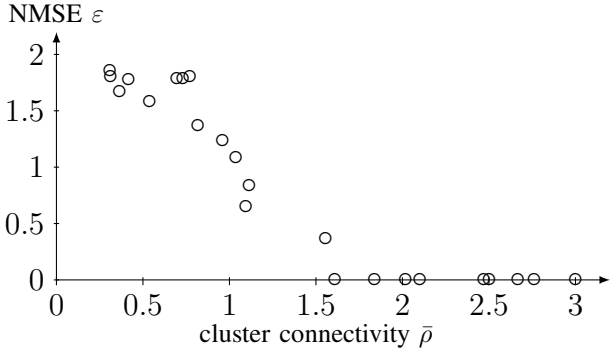
\begin{figure}[htbp]
\begin{center}
\begin{tikzpicture}
 \tikzset{x=2.3cm,y=1.5cm,every path/.style={>=latex},node style/.style={circle,draw}}
    \csvreader[ head to column names,%
                late after head=\xdef\aold{\a}\xdef\bold{\b},,%
                after line=\xdef\aold{\a}\xdef\bold{\b}]%
                {MSEoverBoundary_18-Oct-2019.csv}{}
                {\draw [line width=0.0mm] (\aold, \bold) (\a,\b) node {\large $\circ$};
    }
          \draw[->] (0,0) -- (3.2,0);
 
      \node [] at (1.5,-0.5) {\centering cluster connectivity $\bar{\rho}$};
      \draw[->] (0,0) -- (0,2.2);
      
      \node [anchor=south] at (0,2.2) {NMSE $\varepsilon$};
            \foreach \label/\labelval in {0/$0$,0.5/$0.5$,1/$1$,1.5/$1.5$,2/$2$}
        { 
          \draw (1pt,\label) -- (-1pt,\label) node[left] {\large \labelval};
        }
        
              \foreach \label/\labelval in {0/$0$,0.5/$0.5$,1/$1$,1.5/$1.5$,2/$2$,2.5/$2.5$,3/$3$}
        { 
          \draw (\label,1pt) -- (\label,-2pt) node[below] {\large \labelval};
        }
\end{tikzpicture}
        \vspace*{-4mm}
\end{center}
  \caption{NMSE achieved by Alg.\ \ref{alg_sparse_label_propagation_centralized} for a two-cluster graph 
  (see Fig.\ \ref{fig_clustered_graph_signal}) with varying connectivity $\bar{\rho}$. 
   }
  \label{fig_NMSEconnect}
  \vspace*{-3mm}
\end{figure}

As indicated in Fig.\ \ref{fig_NMSEconnect}, TV minimization fails to recover piece-wise constant graph signals 
$\vx$ of the form \eqref{equ_def_clustered_signal_model} if the cluster connectivity $\bar{\rho}$ is too small. 
We have depicted, for one particular realization of the two-cluster graph $\graph$ with connectivity $\bar{\rho} \approx 0.26$, 
the graph signal estimate obtained from TV minimization via Alg. \ref{alg_sparse_label_propagation_centralized} in Fig.\ \ref{fig_TVminlowcond}. 
Clearly, in this case TV minimization fails to correctly identify the underlying cluster structure and assigns many nodes to the 
signal values of the wrong cluster. We have also included the estimates obtained from nLasso \eqref{equ_nLasso} for different choices 
fo the parameter $\lambda$. According to Fig.\ \ref{fig_TVminlowcond}, the nLasso estimates tend to be forced towards zero while TV minimization 
results in signal values more close to the initial labels $x_{1} = 1/10$ and $x_{200} = -1/10$. 
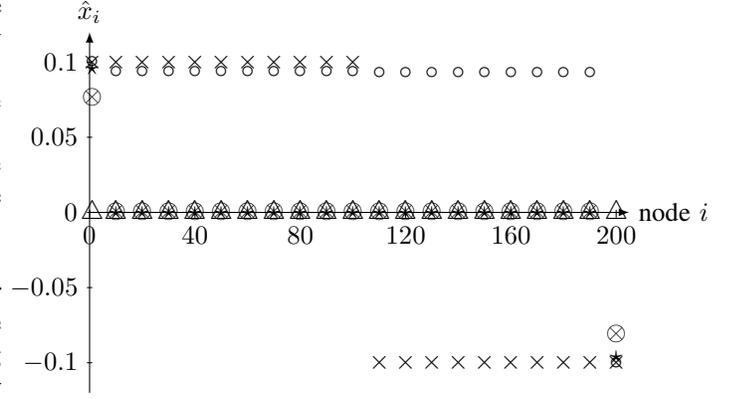
\begin{figure}[htbp]
\hspace*{-5mm}
\begin{tikzpicture}
 \tikzset{x=0.035cm,y=20cm,every path/.style={>=latex},node style/.style={circle,draw}}
    \csvreader[ head to column names,%
                late after head=\xdef\aold{\Node}\xdef\eold{\TVMin}\xdef\fold{\true}\xdef\gold{\lambdavala}\xdef\kold{\lambdavalb}\xdef\mold{\lambdavalc},,%
                after line=\xdef\aold{\Node}\xdef\eold{\TVMin}\xdef\fold{\true}\xdef\gold{\lambdavala}\xdef\kold{\lambdavalb}\xdef\mold{\lambdavalc}]%
                {TVMinNLasso_20-Oct-2019.csv}{}
                {\draw [line width=0mm] (\aold, \eold)  (\Node,\TVMin) node  {$\circ$};
                \draw [line width=0.2mm] (\aold, \fold) (\Node,\true) node {$\times$};
                  \draw [line width=0.2mm] (\aold, \gold) (\Node,\lambdavala) node {$\star$};
                      \draw [line width=0.2mm] (\aold, \kold) (\Node,\lambdavalb) node {$\otimes$};
                                  \draw [line width=0.2mm] (\aold, \mold) (\Node,\lambdavalc) node {$\bigtriangleup$};
    }
          \draw[->] (0,0) -- (205,0);
 
      \node [anchor=west] at (205,0) {node $i$};
      \draw[->] (0,-0.12) -- (0,0.12);
      
      \node [anchor=south] at (0,0.12) {$\hat{x}_{i}$};
            \foreach \label/\labelval in {-0.1/$-0.1$,-0.05/$-0.05$,0/$0$,0.05/$0.05$,0.1/$0.1$}
        { 
          \draw (1pt,\label) -- (-1pt,\label) node[left]  {\labelval};
        }
        
              \foreach \label/\labelval in {0/$0$,	40/$40$,80/$80$,120/$120$,160/$160$,200/$200$}
        { 
          \draw (\label,1pt) -- (\label,-2pt) node[below] { \labelval};
        }
\end{tikzpicture}
        \vspace*{-4mm}
  \caption{Graph signal estimates delivered by Alg.\ \ref{alg_sparse_label_propagation_centralized} for a two-cluster graph 
  (see Fig.\ \ref{fig_clustered_graph_signal}) with connectivity $\bar{\rho}\approx0.26$. True graph signal (``$\times$'') is 
  piece-wise constant over clusters $\cluster_{1}=\{1,2,\ldots,100\}, \cluster_{2}=\{101,102,\ldots,200\}$. 
  The signal is estimated from its values on $\trainingset\!=\!\{1,200\}$ using TV minimization \eqref{equ_min_constr} (``$\circ$'') and 
  nLasso \eqref{equ_nLasso} for $\lambda\!=\!10^{-4}$ (``$\star$``), $\lambda=10^{-3}$ (``$\otimes$'') and $\lambda=10^{-2}$ (``$\bigtriangleup$''). 
   }
  \label{fig_TVminlowcond}
  \vspace*{-3mm}
\end{figure}

\subsection{Stochastic Block Model}
\label{sec_SBM_graph}

In this experiment, we generate an empirical graph $\graph$ using the SBM \cite{AbbeSBM2018}. 
The graph $\graph$ consists of three clusters $\cluster_{1}, \cluster_{2}$ and $\cluster_{3}$, each 
consisting of $10$ nodes. An edge is placed between nodes $i,j$ with probability $p_{\rm in}$ 
if they are in the same cluster and with probability $p_{\rm out}$ if they are from different clusters. 

The empirical graph $\graph$ is then assigned a piece-wise constant 
graph signal $\vx$ (see \eqref{equ_def_clustered_signal_model}) using the partition 
$\partition=\{\cluster_{1},\cluster_{2},\cluster_{3} \}$. We apply Alg.\ \ref{alg_sparse_label_propagation_centralized} 
to recover the graph signal $\vx$ from its values at the nodes in the training set $\trainingset$ 
which contains exactly five nodes from each cluster such that $|\trainingset|\!=\!15$. 

The (non-rigorous) condition \eqref{equ_condition_SBM} suggests that  Alg.\ \ref{alg_sparse_label_propagation_centralized} delivers 
an accurate estimate of $\vx$ whenever $p_{\rm in }/p_{\rm out} \gg (2 / |\trainingset \cap \cluster_{l}| ) (|\nodes| - |\cluster_{l}|)$ 
for all $l=1,2,3$. Inserting the particular SBM parameters used in this experiment yields the condition $p_{\rm in }/p_{\rm out} \gg 8$. 

In Fig.\ \ref{fig_NMSEconnect_SBM}, we depict the normalized mean squared error (NMSE) 
$\varepsilon \defeq \| \vx\!-\!\tilde{\vx}^{(k)} \|^{2}_{2} / \| \vx^{(k)} \|^{2}_{2}$ incurred by 
Alg.\ \ref{alg_sparse_label_propagation_centralized} (averaged over $100$ i.i.d.\ simulation runs) 
for varying ratio $p_{\rm in}/p_{\rm out}$ of SBM edge probabilities $p_{\rm in}, p_{\rm out}$. The 
results in Fig.\ \ref{fig_NMSEconnect_SBM} agree with the (non-rigorous) condition $p_{\rm in }/p_{\rm out} \gg 8$ 
such that TV minimization correctly recovers a piece-wise constant graph signal from few labeled nodes. 

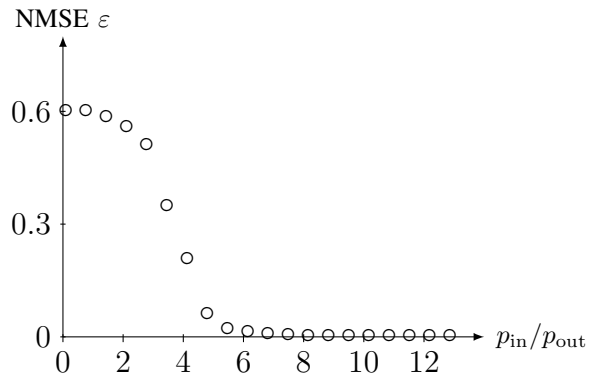
\begin{figure}[htbp]
\begin{center}
\begin{tikzpicture}
 \tikzset{x=0.4cm,y=5cm,every path/.style={>=latex},node style/.style={circle,draw}}
    \csvreader[ head to column names,%
                late after head=\xdef\aold{\a}\xdef\bold{\b},,%
                after line=\xdef\aold{\a}\xdef\bold{\b}]%
                {MSEoverSBMParam_20-Oct-2019.csv}{}
                {\draw [line width=0.0mm] (\aold, \bold) (\a,\b) node {\large $\circ$};
    }
          \draw[->] (0,0) -- (14,0);

      \draw[->] (0,0) -- (0,0.8);
      
      \node [anchor=south] at (0,0.8) {NMSE $\varepsilon$};
            \foreach \label/\labelval in {0/$0$,0.3/$0.3$,0.6/$0.6$}
        { 
          \draw (1pt,\label) -- (-1pt,\label) node[left] {\large \labelval};
        }
        
              \foreach \label/\labelval in {0/$0$,2/$2$,4/$4$,6/$6$,8/$8$,10/$10$,12/$12$}
        { 
          \draw (\label,1pt) -- (\label,-2pt) node[below] {\large \labelval};
        }
         \node [anchor=west] at (14.1,0) {\centering $p_{\rm in}/p_{\rm out}$};
        \vspace*{-3mm}
\end{tikzpicture}
        \vspace*{-3mm}
\end{center}
  \caption{NMSE achieved by Alg.\ \ref{alg_sparse_label_propagation_centralized} for an empirical graph obtained from an SBM 
  with varying edge probabilities $p_{\rm out}$ and $p_{\rm in}$. 
   }
  \label{fig_NMSEconnect_SBM}
  \vspace*{-3mm}
\end{figure}

\subsection{Big Data Framework Implementation}
\label{sec_LFR_graph}

We have implemented Alg.\ \ref{sparse_label_propagation_mp} using the higher-level programming interface 
\textsc{GraphX} \cite{xin2013graphx} for the large-scale distributed computation framework \textsc{Spark} \cite{zaharia2010spark}. 
The central concept of this framework is the distributed data structure (RDD) which is used to represent 
graph nodes, edges and associated signal values. Computations on graph data amount to transformations 
applied to RDDs. These RDD transformations are executed using efficient low-level distributed computing primitives \cite{zaharia2010spark}. 

Using this implementation, we applied \ref{sparse_label_propagation_mp} to synthetic data obtained from the 
Lancichinetti-Fortunato-Radicchi (LFR) network model \cite{PhysRevE.78.046110}. The probabilistic LFR model is widely used for 
benchmarking network algorithms \cite{PhysRevE.78.046110} and aims at imitating some key characteristics of ``real-world'' 
networks such as the internet \cite{NewmannBook}. 


In order to study the scalability of Alg.\ \ref{sparse_label_propagation_mp}, we generated empirical 
graphs (using the LFR model) of varying size. We then measured the execution time of Alg.\ \ref{sparse_label_propagation_mp} 
for a fixed number of $100$ iterations. 

As indicated by Fig.\ \ref{fig_scalability1}, the execution time scales linearly with the size (number of nodes) 
of the empirical graph. Fig.\ \ref{fig_scalability1} also illustrates the effect of adding worker nodes to the cluster. In particular, for 
an empirical graph with size $|\nodes|\!=\!10^5$, we determined the execution time of Alg.\ \ref{sparse_label_propagation_mp} 
when the number of worker nodes is increased from $1$ up to $8$. As expected, the execution time decreases with 
increasing number of worker nodes. This decrease in execution time is, however, not exactly proportional to the increase of 
worker nodes due to communication overhead and data fragmentation associated with parallel computation frameworks \cite{Amdahl67}.

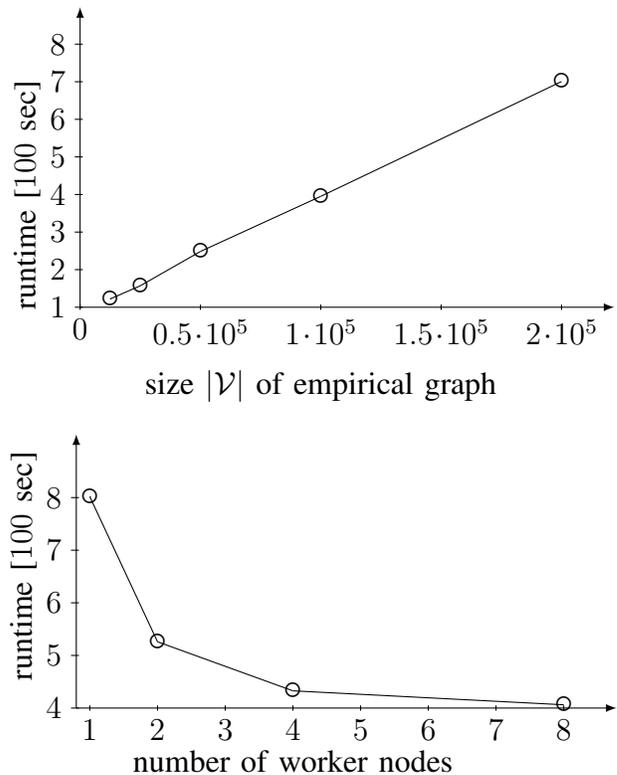
\begin{figure}[htbp]
\hspace*{2mm}
\begin{minipage}[t][][t]{\columnwidth}
\vskip0pt
 \begin{tikzpicture}
   \newcommand{\ymax}{8}
   \newcommand{\ymin}{1} 
    \newcommand{\xmin}{0} 
   \newcommand{\xmax}{2} 
 \tikzset{x=3.2cm,y=0.5cm,every path/.style={>=latex},node style/.style={circle,draw}}
    \csvreader[ head to column names,%
                late after head=\xdef\aold{\a}\xdef\bold{\b},%
                after line=\xdef\aold{\a}\xdef\bold{\b}]%
                {Exectime.csv}{}{%
    \draw  (\aold, \bold) -- (\a,\b)  node {\Large $\circ$};
   }
   \draw[->] (\xmin,\ymin) -- ([xshift=20pt]\xmax,\ymin); 
   \node [] at ([yshift=-30pt]\xmax/2,\ymin) {\centering \large size $|\nodes|$ of empirical graph};
  \draw[->] (\xmin,\ymin) -- ([yshift=14pt]\xmin,\ymax) ;  
  \node [rotate=90] at ([xshift=-20pt]\xmin,\ymax/2) {\centering \large runtime [$100$ sec]};
            \foreach \label/\labelval in {1/$1$,2/$2$,3/$3$,5/$5$,4/$4$,6/$6$,7/$7$,8/$8$}  
        { 
          \draw (1pt,\label) -- (-2pt,\label) node[left] {\large \labelval};
        }
              \foreach \label/\labelval in {0/$0$,0.5/$0.5\!\cdot\!10^5$,1/$1\!\cdot\!10^5$,1.5/$1.5\!\cdot\!10^5$,2/$2\!\cdot\!10^5$}    
        { 
          \draw ([yshift=1pt]\label,\ymin) -- ([yshift=-1pt]\label,\ymin) node[below] {\large \labelval};
        }
\end{tikzpicture}
\end{minipage}

\vspace*{1mm}

\hspace*{1.5mm}
\begin{minipage}[t][][t]{\columnwidth}
\vskip0pt
\vspace*{2mm}
 \begin{tikzpicture}
   \newcommand{\ymax}{8.5}
   \newcommand{\ymin}{4} 
    \newcommand{\xmin}{0.8} 
   \newcommand{\xmax}{8} 
 \tikzset{x=0.9cm,y=0.7cm,every path/.style={>=latex},node style/.style={circle,draw}}
    \csvreader[ head to column names,%
                late after head=\xdef\aold{\a}\xdef\bold{\b},%
                after line=\xdef\aold{\a}\xdef\bold{\b}]%
                {ExectimeWorkers.csv}{}{%
    \draw  (\aold, \bold) -- (\a,\b)  node {\Large $\circ$};
   }
   \draw[->] (\xmin,\ymin) -- ([xshift=20pt]\xmax,\ymin); 
   \node [] at ([yshift=-20pt]\xmax/2,\ymin) {\centering  \large number of worker nodes};
  \draw[->] (\xmin,\ymin) -- ([yshift=14pt]\xmin,\ymax) ;  
  \node [rotate=90,right] at ([xshift=-20pt,yshift=\ymin/2]\xmin,\ymax/2) {\centering \large runtime [$100$ sec]};
            \foreach \label/\labelval in {5/$5$,4/$4$,6/$6$,7/$7$,8/$8$}  
        { 
          \draw ([xshift=1pt]\xmin,\label) -- ([xshift=-2pt]\xmin,\label) node[left] {\large \labelval};
        }
              \foreach \label/\labelval in {1/$1$,2/$2$,3/$3$,4/$4$,5/$5$,6/$6$,7/$7$,8/$8$}    
        { 
          \draw ([yshift=1pt]\label,\ymin) -- ([yshift=-1pt]\label,\ymin) node[below] {\large \labelval};
        }
\end{tikzpicture}
\end{minipage} \caption{(Top) Runtime of Alg.\ \ref{sparse_label_propagation_mp} as size $|\nodes|$ of empirical graph increases. 
(Bottom) Runtime for varying number of ``worker nodes''. }
    \label{fig_scalability1}
\end{figure}

\subsection{Road Network}
\label{sec_road_network}

In this experiment we consider a dataset with empirical graph $\graph_3=(\nodes,\edges,\mathbf{W})$ representing 
a road network in North Jutland (Denmark) 
\cite{Lichman:2013,KaulYangJensen2013}. 
The edges $\edges$ of the graph $\graph_3$ represent segments of road, and the nodes $\nodes$ are intersections or terminations of roads. 
The empirical graph $\graph_3$ contains  $\nrnodes \approx 4 \cdot 10^5$ nodes and 
$\nredges \approx 3.7 \cdot 10^5$ edges. The edge weights $W_{i,j}$ are obtained from the great-circle distances between 
intersections, measured in kilometres. 

Each node $i\!\in\!\nodes$ of $\graph_{3}$ is labeled with the elevation $x_{i}\!\in\!\mathbb{R}$ 
(relative to sea level) of the corresponding location in the road network. 
We construct a training set $\trainingset$ by selecting $|\nodes|/10$ nodes of  
$\graph_3$ uniformly at random. 
Based on the labels of the nodes in the training set, we recover (predict) the 
labels on the remaining nodes using Alg.\ \ref{alg_sparse_label_propagation_centralized}, 
nLasso \eqref{equ_nLasso} and LP \eqref{equ_LP_problem}. 
The results are presented in Fig.\ \ref{fig_conv_road}, 
which depicts the NMSE achieved by the different algorithms 
after a certain number $k$ of iterations (the iterations of the three methods having similar 
computational complexity). 

As indicated by Fig.\ \ref{fig_conv_road}, TV minimization Alg.\ \ref{alg_sparse_label_propagation_centralized} 
converges rapidly to a solution with smaller NMSE than nLasso (with manually tuned $\lambda$ in \eqref{equ_nLasso}) 
and LP \eqref{equ_LP_problem}. 



\begin{figure}[htbp]
 \begin{tikzpicture}
 \newcommand{\ymin}{0}
  \newcommand{\ymax}{1}
 \tikzset{x=0.025cm,y=4.5cm,every path/.style={>=latex},node style/.style={circle,draw}}
    \csvreader[ head to column names,%
                late after head=\xdef\aold{\a}\xdef\bold{\SLP}\xdef\cold{\LP}\xdef\dold{\NWL},%
                after line=\xdef\aold{\a}\xdef\bold{\SLP}\xdef\cold{\LP}\xdef\dold{\NWL}]%
            {RoadNetResults.csv}{}{
    \draw  (\aold, \bold) -- (\a,\SLP)  node {$\times$};
    \draw (\aold, \cold) -- (\a,\LP)  node {$\circ$}; 
    \draw (\aold, \dold) -- (\a,\NWL) ; 
   }
   \node[ align=left,   below] at (150,0.8)   {{\large $-$ nLasso \eqref{equ_nLasso}}\\[1mm]{\large $\circ$ LP \eqref{equ_LP_problem}}\\[1mm]{\large $\times$ Alg.\ \ref{alg_sparse_label_propagation_centralized}}} ; 
   \draw[->] (0,\ymin) -- (320,\ymin); 
   \node [] at ([yshift=-30pt]150,\ymin) {\centering \large number $k$ of iterations};
  \draw[->] (0,\ymin) -- (0,\ymax) ;  
  \node [anchor=south west] at (-5,1) {$\varepsilon \times 10$ };
            \foreach \label/\labelval in {0/$0$,0.1/$1$,0.2/$2$,0.3/$3$,0.4/$4$,0.5/$5$,0.6/$6$,0.7/$7$,0.8/$8$,0.9/$9$}
        { 
          \draw (1pt,\label) -- (-5pt,\label) node[left] {\large \labelval};
        }
              \foreach \label/\labelval in {1/$1$,50/$50$,100/$100$,150/$150$,200/$200$,250/$250$,300/$300$}
        { 
          \draw ([yshift=1pt]\label,\ymin) -- ([yshift=-1pt]\label,\ymin) node[below] {\large \labelval};
        }
\end{tikzpicture}
\vspace*{-2mm}
 \caption{NMSE $\varepsilon$ incurred by learning methods applied to empirical graph $\graph_3$ for increasing number of iterations. 
 }
\label{fig_conv_road}
\end{figure}
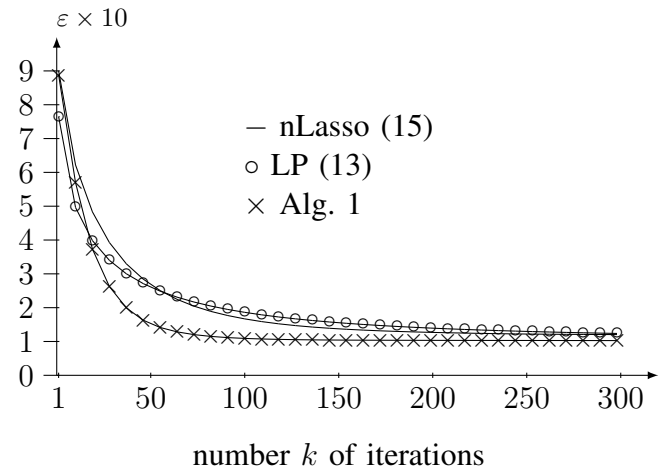

\section{Conclusion}
\label{sec5_conclusion}

We have offered an analysis of the computational and statistical properties of TV minimization 
from a network flow perspective. Using a network flow perspective allowed us to derive conditions 
on network structure and available label information such that TV minimization accurately learns 
piece-wise constant graph signals. We have also obtained a novel interpretation of primal-dual 
methods for TV minimization as distributed methods for network (flow) optimization. 

Several topics for follow-up research can be identified. First, we plan to extend our network-flow 
based analysis of TV minimization to the closely related nLasso problem. This seems to be quite 
straightforward and might require merely a minor modification of the network flow constraints used 
to measure cluster connectivity. Regarding computational properties of TV based methods, we consider 
extending work on partial linear convergence of non-smooth Lasso type problems to TV minimization and 
nLasso. It turns out that such problems can be solved by iterative methods that converge linearly 
(at a geometric rate) up to a sub-optimality on the order of the intrinsic estimation error, which cannot be overcome 
by any algorithm. Using the duality of TV minimization and network flows, such results would have immediate 
consequences for the complexity of network flow (and clustering) problems. As to the statistical properties of TV 
minimization, it would be interesting to extend our analysis from piece-wise constant to piece-wise smooth 
graph signals. 

We expect our work to initiate cross-fertilization between network science and compressive 
graph signal processing. Convex methods for TV minimization are computationally 
attractive methods for handling massive networks and it would be interesting to investigate if they 
might outperform state-of-the art network algorithms in some settings. 

On the other hand, the duality of TV minimization and network flow optimization suggests new routes 
for combining primal-dual methods for TV minimization with existing methods for clustering and computing 
(approximating) maximum network flows. In particular, we might use maximum flow methods to (approximately) 
solve the dual of TV minimization in order to obtain an initial solution for TV minimization via the primal-dual 
optimality condition presented in Section \ref{sec_dual_TV_Min}. The initial estimates for the solutions of the 
primal and dual problem might then, in turn, be used to warm-start the primal-dual iterations underlying Alg.\ \ref{alg_sparse_label_propagation_centralized}. 

\vspace*{-2mm}
\section*{Acknowledgement}
We would like to acknowledge support from the Vienna Science Fund (WWTF) Grant ICT15-119 and US ARO grant W911NF-15-1-0479.

\vskip 0.2in
\bibliographystyle{IEEEtran}
\bibliography{/Users/junga1/Literature.bib}

\onecolumn

\end{document}